\def\year{2022}\relax
\documentclass[letterpaper]{article} 
\usepackage{aaai22}  
\usepackage{times}  
\usepackage{helvet}  
\usepackage{courier}  
\usepackage[hyphens]{url}  
\usepackage{graphicx} 
\usepackage{natbib}  
\usepackage{caption} 
\DeclareCaptionStyle{ruled}{labelfont=normalfont,labelsep=colon,strut=off} 
\usepackage{algorithm}
\usepackage{algorithmic}

\usepackage{newfloat}
\usepackage{listings}
\floatstyle{ruled}
\newfloat{listing}{tb}{lst}{}
\floatname{listing}{Listing}
\nocopyright
\usepackage{amsmath}
\usepackage{multirow}
\usepackage{subcaption}
\usepackage{tikz}
\usepackage{xcolor}
\usepackage{dsfont} 
\usepackage{mathtools}

\definecolor{marsala}{rgb}{0.58, 0.27, 0.29}
\definecolor{brightmaroon}{rgb}{0.76, 0.13, 0.28}
\definecolor{chromeyellow}{rgb}{1.0, 0.65, 0.0}
\definecolor{mintgreen}{rgb}{0.6, 1.0, 0.6}

\usepackage{amssymb}
\usepackage{amsthm}

\newtheorem{theorem}{Theorem}[section]

\newtheorem{remark}{Remark}
\theoremstyle{definition}
\newtheorem{definition}{Definition}
\newtheorem{lemma}[definition]{Lemma}
\newtheorem{proposition}[definition]{Proposition}

\newcommand{\newron}{{\sc Newron}}
\newcommand{\ian}{{\sc IAN}}

\title{\newron: A New Generalization of the Artificial Neuron to Enhance the Interpretability of Neural Networks}
\author {
    Federico Siciliano, \textsuperscript{\rm 1}
    Maria Sofia Bucarelli, \textsuperscript{\rm 1}
    Gabriele Tolomei, \textsuperscript{\rm 2}
    Fabrizio Silvestri \textsuperscript{\rm 1} \\
}
\affiliations {
        \textsuperscript{\rm 1} Department of Computer, Control and Management Engineering Antonio Ruberti, Sapienza University of Rome, Italy \\
        \textsuperscript{\rm 2} Department of Computer Science, Sapienza University of Rome, Italy \\
        federico.siciliano@uniroma1.it, mariasofia.bucarelli@uniroma1.it, tolomei@di.uniroma1.it, fsilvestri@diag.uniroma1.it
    }

\begin{document}

\maketitle


\begin{abstract}

In this work, we formulate \newron: a generalization of the McCulloch-Pitts neuron structure. This new framework aims to explore additional desirable properties of artificial neurons. We show that some specializations of \newron\ allow the network to be interpretable with no change in their expressiveness. By just inspecting the models produced by our \newron-based networks, we can understand the rules governing the task. Extensive experiments show that the quality of the generated models is better than traditional interpretable models and in line or better than standard neural networks.

\end{abstract}

\section{Introduction}
\label{sec:intro}

Neural Networks (NNs) have now become the \textit{de facto} standard in most Artificial Intelligence (AI) applications. The world of Machine Learning has moved towards Deep Learning, i.e., a class of NN models that exploit the use of multiple layers in the network to obtain the highest performance.

Research in this field has focused on methods to increase the performance of NNs, in particular on which activation functions \cite{apicella2021survey} or optimization method \cite{sun2019survey} would be best.
Higher performances come at a price: \cite{arrieta2020explainable} show that there is a trade-off between interpretability and accuracy of models. Explainable Artificial Intelligence (XAI) is a rapidly growing research area producing methods to interpret the output of AI models in order to improve their robustness and safety (see e.g. \cite{ghorbani2019interpretation} and \cite{bhatt2019building}). Deep Neural Networks (DNNs) offer the highest performance at the price of the lowest possible interpretability. It is an open challenge to attain such high performance without giving up on model interpretability.

The simplest solution would be to use a less complex model that is natively interpretable, e.g., decision trees or linear models, but those models are usually less effective than NNs. We ask the following question: can we design a novel neural network structure that makes the whole model interpretable without sacrificing effectiveness?

NNs are black-box models: we can only observe their input and output values with no clear understanding of how those two values are correlated according to the model's parameters. Although a single neuron in the NN performs a relatively simple linear combination of the inputs, there is no clear and straightforward link between the parameters estimated during the training and the functioning of the network, mainly because of the stacking of multiple layers and non-linearities.

In this work, we propose a generalization of the standard neuron used in neural networks that can also represent new configurations of the artificial neuron. Thus, we discuss a specific example that allows us to interpret the functioning of the network itself.

We focus our efforts on tabular data since we investigate how \newron\ works only in the case of fully connected NNs. It is more straightforward to produce human-readable rules for this kind of data. We also remark that our goal is not to improve the performance of NNs, but rather to create interpretable versions of NNs that perform as well as other interpretable models (e.g., linear/logistic regression, decision trees, etc.) and similarly to standard NNs, when trained on the same data.

\subsection{Motivating Example}
\label{sec:motivations}

\begin{figure*}[ht]
\centering
\includegraphics[width=0.9\textwidth]{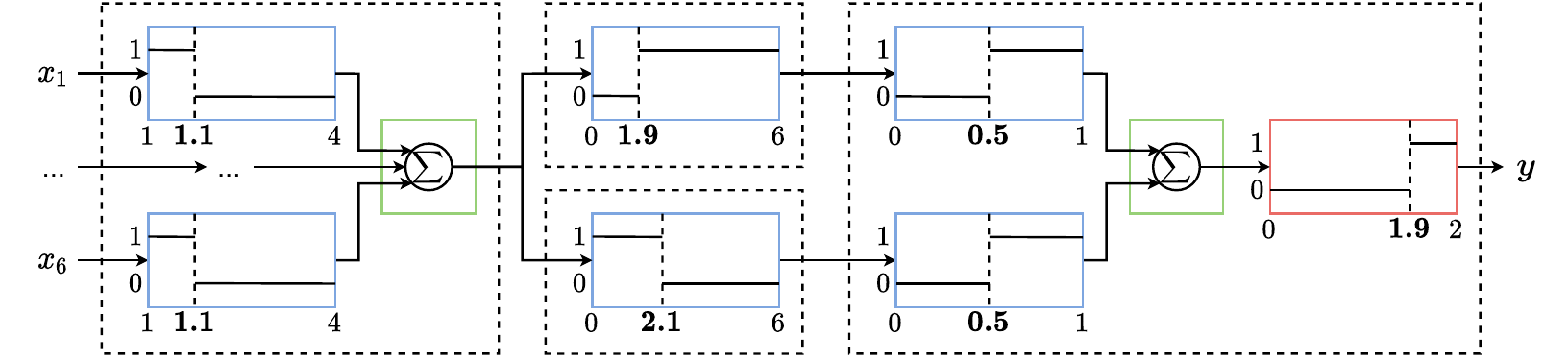}
\caption{An example of a network for the MONK-2 dataset. $x_i$ are the inputs, $y$ is the output. The red and blue rectangles represent the plot of functions, with input range on the $x$-axis and output on the $y$-axis. The green rectangles contain the aggregation function. The numbers in bold represent the thresholds for the step functions.}
\label{img: monk-2 network}
\end{figure*}

Consider a simple dataset: MONK's\footnote{https://archive.ics.uci.edu/ml/datasets/MONK\%27s+Problems}. Each sample consists of $6$ attributes, which take integer values between $1$ and $4$ and a class label determined by a decision rule based on the $6$ attributes. For example, in  MONK-2, the rule that defines the class for each sample is the following: ``\textit{exactly two}'' out of the six attributes are equal to $1$.

It is impossible to intuitively recover rules from the parameter setting from a traditional, fully connected NN.

We shall see in the following that our main idea is that of inverting the activation and aggregation. In \newron\, the nonlinearity directly operates on the input of the neuron. The nonlinearity acts as a thresholding function to the input, making it directly interpretable as a (fuzzy) logical rule by inspecting its parameters. Consider the following network, represented in Figure \ref{img: monk-2 network}: $2$ hidden layers, the first with $1$ neuron, the second with $2$ neurons, and $1$ output neuron. The $x_i$'s are the inputs of the model, $y$ is the output.





We present the form of a typical architecture composed by \newron\ in Figure \ref{img: monk-2 network}. We show how we can interpret the parameters obtained from a trained network. The rectangles represent the plot of a function that divides the input domain into two intervals, separated by the number below the rectangle, taking values $1$ and $0$.

The functions that process the input give output $1$ only if the input is less than $1.1$, given that inputs are integers and assume values only in $\{1,2,3,4\}$, this means ``if $x_i = 1$''. The sum of the output of all these functions, depicted in the green rectangle, then represents the degree of soundness of those rules are.

The second layer has two neurons: the first outputs $1$ if it receives an input greater than $1.9$, i.e. if at least $2$ of the rules $x_i = 1$ are valid, while the second outputs $1$ if it receives an input less than $2.1$, i.e. if $2$ or less of the rules $x_i = 1$ are valid. Notice that the two neurons are activated simultaneously only if $x_i = 1$ is true for exactly two attributes.

In the last layer, functions in the blue rectangles receive values in $\{0,1\}$ and do not operate any transformation, keeping the activation rules unchanged. The sum of the outputs of these functions is then passed to the function in the red rectangle. This function outputs $1$ only if the input is greater than $1.9$. Since the sum is limited in ${0,1,2}$, this happens only when it receives $2$ as input, which occurs only if the two central neurons are activated. As we have seen, this only applies if exactly $2$ of the rules $x_i = 1$ are valid.

So we can conclude that the network gives output $1$ just if ``\textit{exactly two}'' of $\{x_1 = 1, x_2 = 1, x_3 = 1, x_4 = 1, x_5 = 1, x_6 = 1\}$ are true.

\subsection{Contributions}
\label{sec:contributions}
The main contributions of this work are the following:
\begin{itemize}
    \item We propose \newron, a generalization of the McCulloch-Pitt neuron allowing the definition of new artificial neurons. We show how special cases of \newron\ may pave the way towards interpretable, white-box neural networks.
    \item We prove the universal approximation theorem for three specializations of \newron, demonstrating that the new model does not lose any representation power in those cases.
    \item We experiment on several tabular datasets showing that \newron\ allows learning accurate Neural models, beating interpretable by design models such as Decision Trees and Logistic Regression.
\end{itemize}

\section{Related Work}
\label{sec:related}

\cite{Rosenblatt58} introduced the single artificial neuron: the Perceptron. The Perceptron resembles the functioning of the human/biological neuron, where the signal passing through the neuron depends on the intensity of the received signal, the strength of the synapses, and the receiving neuron's threshold. In the same way, the Perceptron makes a linear combination of the inputs received and is only activated if the result exceeds a certain threshold.
Over the years, various improvements to neural networks have been proposed: Recurrent Units, Convolutional Layers, and Graph Neural Networks, but for Fully Connected NNs, research efforts have mainly focused on finding more efficient activation functions \cite{apicella2021survey}.
Two works that have focused on modifying the internal structure of the neuron are those of \cite{kulkarni2009generalized}, and \cite{fan2018new}. In the former, a neuron is introduced that performs both a sum and a product of the inputs in parallel, applies a possibly different activation function for the two results, and then sums the two outcomes. Despite promising results, given the use of fewer parameters, better performance, and reduced training time compared to standard MLPs and RNNs, the proposed neuron, rather than being a generalization, is a kind of union between two standard neurons, one of which uses the product, instead of sum, as aggregation function. 
In the second paper, starting from the notion that the traditional neuron performs a first-order Taylor approximation, the authors propose a neuron using a second-order Taylor approximation. Although this improves the capacity of a single neuron, the authors do not demonstrate any gains in terms of training time or convergence. Indeed, this can be considered a particular case of the higher-order neural units (HONUs) (see, e.g., \cite{gupta2013fundamentals}), i.e., a type of neurons that, by increasing the degree of the polynomial computed within them, try to capture the higher-order correlation between the input patterns.
Recent works that focus on interpretation at neuron level (\cite{dalvi2019one}, \cite{dalvi2019neurox}, \cite{heo2019knowledge}, \cite{nam2020relative}) often concentrate on extracting the most relevant neurons for a given task, but mostly deal with Recurrent or Convolutional neural networks.
Although not designing an alternative version of the neuron, \cite{yang2018deep} proposes an alternative neural network structure, based on a Binning Layer, which divides the single input features into several bins, and a Kronecker Product Layer, which takes into account all the possible combinations between bins. The parameters estimated during training can be interpreted to translate the network into a decision tree through a clever design of the equations defining the network. Although interpretable, the main issue in this work is its scalability. The Kronecker Product Layer has an exponential complexity that makes training time unfeasible when the number of features grows.

\section{The \newron\ Structure}
\label{sec:newron}

A neuron, in the classical and more general case, is represented by the equation $y = f\left(b + \sum_{i=1}^n w_i x_i\right)$.

\begin{figure}[h]
\centering
\includegraphics[width=0.9\columnwidth]{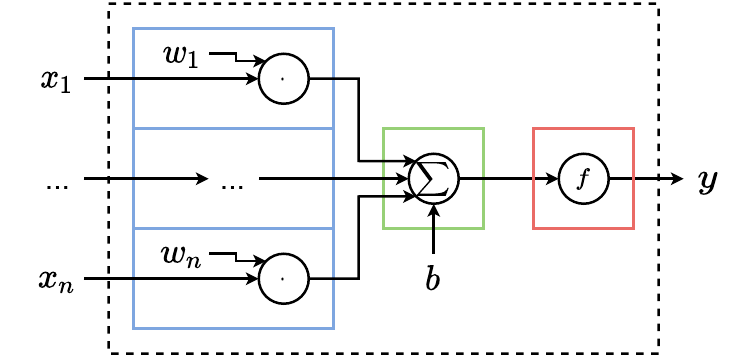}
\caption{Structure of the standard artificial neuron. $w_i$ and $b$ are respectively weights and bias. $f$ is the activation function. $x_i$'s are the inputs and $y$ is the output.}
\label{fig:standard_neuron}
\end{figure}

$b$ is called the bias, $w_i$ are the weights, and $x_i$s are the inputs. $f$ represents the activation function of the neuron. Usually, we use the sigmoid, hyperbolic tangent, or ReLU functions.

We first generalize the above equation, introducing \newron\ as follows:

\begin{equation}
\label{eq:gen-neuron}
    y = f\left(G_{i=1}^n\left(h_i(x_i)\right)\right)
\end{equation}

\begin{figure}[h]
\centering
\includegraphics[width=0.9\columnwidth]{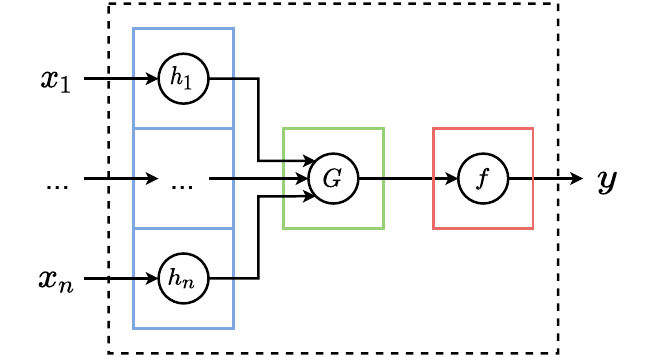}
\caption{Structure of \newron, the generalized artificial neuron. The blue rectangles represent the processing function sections, the green rectangles contain the aggregation function, and the red rectangles represent the activation part. Same colors are also used in Figure \ref{fig:standard_neuron}}
\label{fig:newron}
\end{figure}

Each input is first passed through a function $h_i$, which we will call \textit{processing function}, where the dependence on $i$ indicates different parameters for each input. $G$, instead, represents a generic aggregation function.

Using \newron\ notation, the standard artificial neuron would consist of the following: $h_i(x_i) = w_ix_i$, $G = \sum_{i=1}^n$, and $f(z) = f^*(z+b)$.


$G$ does not have any parameters, while $b$ parametrizes the activation function.

\subsection{Inverted Artificial Neuron (\ian)}

We present $3$ novel structures characterized by an inversion of the aggregation and activation functions. We name this architectural pattern: Inverted Artificial Neuron (\ian). In all the cases we consider the sum as the aggregation function and do not use any activation function: $G = \sum$, and  $f(z) = z$.

\subsubsection{Heaviside \ian}

The first case we consider uses a unit step function as activation. This function, also called the Heaviside function, is expressed by the following equation:

\begin{equation}
    H(x) = 
    \begin{cases}
        1 & x\geq 0 \\
        0 & x<0
    \end{cases}
\end{equation}

According to (\ref{eq:gen-neuron}) we can define the processing function as follows:

\begin{equation}
    h(x_i) = H(w_i(x_i-b_i)) =  
    \begin{cases}
        H(w_i) & x_i\geq b_i \\
        1-H(w_i) & x_i<b_i
    \end{cases}
\end{equation}

where $w_i$ and $b_i$ are trainable parameters.

\subsubsection{Sigmoid \ian}

We cannot train the Heaviside function using gradient descent, and it represents a decision rule that in some cases is too restrictive and not ``fuzzy'' enough to deal with constraints that are not clear-cut.

A natural evolution of the unit step function is therefore the sigmoid function $\sigma(x) = \frac{1}{1+e^{-x}}$. This function ranges in the interval $(0,1)$, is constrained by a pair of horizontal asymptotes, is monotonic and has exactly one inflection point.

The sigmoid function can be used as a processing function with the following parameters: $h(x_i) = \sigma(w_i(x_i-b_i))$.

\subsubsection{Product of $\tanh$ \ian}

Another option we consider as a processing function is the multiplication of hyperbolic tangent ($\tanh$). For simplicity, we will use the term ``$\tanh$-prod''.

The $\tanh$ function $\tanh(x) = \frac{e^{2x}-1}{e^{2x}+1}$ is on its own very similar to the sigmoid. An interesting architecture is that using $M$ $\tanh$ simultaneously. Each $\tanh$ applies its own weights, on each individual input.

While the sigmoid is monotonic with only one inflection point, roughly dividing the input space into two sections, the multiplication of $\tanh$, by being not monotonic, allows us to divide the input space into several intervals. The multiplication would remain in $(-1,1)$, but can be easily rescaled to $(0,1)$.

We can therefore write the processing function in the case of the $\tanh$ multiplication as follows:

\begin{equation}
\label{rescaled_prod_tanh}
    h(x_i) = \frac{\left(\prod_{m=1}^M\tanh(w_{im}(x_i-b_{im}))\right)+1}{2}
\end{equation}

Note how, in this case, the weights depend on both the input $i$ and the $m$-th function. Such a neuron will therefore have $M$ times more parameters than the Heaviside and sigmoid cases.

\subsubsection{Output layer}

The output layer would produce values ranging in the interval $(0,N)$ ($\{0,1,...,N\}$ for the Heaviside case), where $N$ represents the number of neurons in the penultimate layer. This is because the last neuron makes the sum of $N$ processing functions restricted in the interval $(0,1)$ ($\{0,1\}$ for the Heaviside case). To allow the last layer to have a wider output range and thus make our network able to reproduce a wider range of functions, we modify the last layer processing function $h^*$ as follows: $h^*(x_i) = \alpha_i h(x_i)$,



where $\alpha_i$ are trainable parameters.

In the same way, as for a traditional neural network, it is important, in the output layer, to choose an adequate activation function. We need, indeed, to match the range of the output of the network and the range of the target variable. In particular, in the case of output in $(0,1)$, we use a sigmoid centered in $b^*$: $f^*(z) = \sigma(z-b^*)$

In the case of a classification problem with more than $2$ classes, a softmax function ($s(z_j) = \frac{e^{z_j}}{\sum_l e^{z_l}}$) is used to output probabilities.

\subsubsection{Note(s)}

The writing $w(x-b)$ is theoretically identical to that $w^*x+b^*$, where simply $w^*=w$ and $b^* = -bw$. This notation allows us to interpret the weights directly. From $b$, we already know the inflection point of the sigmoid; while looking at $w$, we immediately understand its direction.

\section{Interpretability}
\label{sec:interpretability}

\cite{arrieta2020explainable} presented a well-structured overview of concepts and definitions in the context of Explainable Artificial Intelligence (XAI). 

They make a distinction among the various terms that are mistakenly used as synonyms for interpretability. According to them:

\begin{itemize}
    \item \textbf{Interpretability:} is seen as a passive feature of the model and represents the ability of a human to understand the underlying functioning of a decision model, focusing more on the cause-effect relationship between input and output.
    
    \item \textbf{Transparency:} very similar to interpretability, as it represents the ability of a model to have a certain degree of interpretability. There are three categories of transparency, representing the domains in which a model is interpretable. Simulatable models can be emulated even by a human. Decomposable models must be explainable in their individual parts. For algorithmically transparent models, the user can understand the entire process followed by an algorithm to generate the model parameters and how the model produces an output from the input.
    
    \item \textbf{Explainability:} can be seen as an active feature of a model, encompassing all actions that can detail the inner workings of a model. The explanation represents a kind of interface between a human and the model and must at the same time represent well the functioning of the model and be understandable by humans.
\end{itemize}

In this paper, we show decomposable models that, in some cases, are also algorithmically transparent.

\subsection{Heaviside}
The interpretability of an architecture composed of Heaviside \ian s has to be analyzed by discussing its four main sections separately.

\subsubsection{First layer - Processing function}

A single processing function $h(x) = H(w(x-b))$ divides the space of each variable $x$ in two half-lines starting from $b$, one of which has a value of $1$ and one of which has a value of $0$, depending on the sign of $w$.

\subsubsection{Aggregation}

Using sum as the aggregation function, the output takes values in $\{0,1,...,n\}$; where $0$ corresponds to a deactivation for each input, and $n$ represents an activation for all inputs, and the intermediate integer values $\{1,2,...k,...,n-1\}$ represent activation for $k$ of inputs.

\begin{equation}
    y = \sum_{i=1}^{n} h_i =
    \begin{cases}
        n & h_i=1 \text{ } \forall i \in \{1,...,n\}\\
        k & h_i=1 \textbf{ } i \in S \subseteq \{1,...,n\},  |S| = k\\
        0 & h_i=0 \text{ } \forall i \in \{1,...,n\}\\
    \end{cases}
\end{equation}
where we simplified the notation using $h_i = h\left(x_i\right)$.

\subsubsection{2+ Layer - Processing function}
Let us define an $M$-of-$N$ rule as true if at least $M$ of the $N$ rules of a given set are true.

The Heavisides of the layers after the first one receive values in $\{0,1,...,n\}$, where $n$ represents the number of inputs of the previous layer.
In the case where $0 \leq b \leq n$ and $w > 0$, the Heaviside will output $1$ only if the input received is greater than or equal to $b$, therefore only if at least $\lceil b \rceil$ of the rules $R_i$ of the previous layer are true, which corresponds to a rule of the type $\lceil b\rceil-of-\{R_1,R_2,...,R_n\}$. 
In the opposite case, where $0 \leq b \leq n$ and $w < 0$, Heaviside will output $1$ only if the input received is less than or equal to $b$, so only if no more than $\lfloor b \rfloor$ of the rules of the previous layer are true. This too can be translated to an $M$-of-$N$ rule, inverting all rules $R_j$ and setting $M$ as $\lceil n-b_i \rceil$: $\lceil n-b_i \rceil-of-\{\neg R_1, \neg R_2,..., \neg R_n\}$.

\subsubsection{Last layer - Aggregation}

In the last layer we have to account for the $\alpha$ factors used to weigh the contribution of each input:

\begin{equation}
    y = \sum_{i=1}^{n} \alpha_i h_i(x_i) = \sum_{i=1}^{n} \alpha_i H(w_i (x_i-b_i))
\end{equation}

We have an activation rule for each of the $n$ Heavisides forcing us to calculate all the $2^n$ possible cases. The contribution of each input is exactly $\alpha_i$. So, the output corresponds to the sum of the $\alpha_i$'s for each subset of inputs considered.

\subsection{Sigmoid}

In the case of sigmoid \ian, $b_i$ represents the inflection point of the function, while the sign of $w_i$ tells us in which direction the sigmoid is oriented; if positive, it is monotonically increasing from $0$ to $1$, while if negative, it is monotonically decreasing from $1$ to $0$. The value of $w_i$ indicates how fast it transitions from $0$ to $1$, and if it tends to infinity, the sigmoid tends to the unit step function.

\subsubsection{Sigmoid Interpretation}

The sigmoid can be interpreted as a fuzzy rule of the type $x_i > b_i$ if $w_i>0$ or $x_i < b_i$ if $w_i<0$, where the absolute value of $w_i$ indicates how sharp the rule is. The case $w_i=0$ will always give value $0.5$, so that the input does not have any influence on the output.

If $w_i$ is very large, the sigmoid tends to the unit step function. If, on the other hand, $w_i$ takes values for which the sigmoid in the domain of $x_i$ resembles a linear function, what we can say is that there is a direct linear relationship (or inverse if $w_i<0$) with the input.

The fuzzy rule can be approximated by its stricter version $x_i > b_i$, interpreting fall under the methodology seen for Heaviside. However, this would result in an approximation of the operation of the network.

It is more challenging to devise clear decision rules when we add more layers. Imagine, as an example, a second layer with this processing function:

\begin{equation}
    h(y) = \sigma(w^*(y-b^*))
\end{equation}

where $y$ is the aggregation performed in the previous layer of the outputs of its processing functions, its value roughly indicates how many of the inputs are active. In the second layer, consider as an example a value of $w^*>0$. To have an activation, this means that we might need $k$ inputs greater than or equal to $b^*/k$. Although this does not deterministically indicate how many inputs we need to be true, we know how the output changes when one of the inputs changes.

The last case to consider takes into account the maximum and minimum values that the sigmoid assumes in the domain of $x$. If they are close to each other, that happens when $w$ is very small, the function is close to a constant bearing no connection with the input.

\subsection{Product of $\tanh$}

The multiplication of $\tanh$ has more expressive power, being able to represent both what is represented with the sigmoid, as well as intervals and quadratic relations.

\subsubsection{$\tanh$-prod Interpretation}

In this case, it is not possible to devise as quickly as in the previous case decision rules. Indeed, it is still possible to observe the trend of the function and draw some conclusions.
When the product of the two $\tanh$ resembles a sigmoid, we can follow the interpretation of the sigmoid case. In other cases, areas with quadratic relations can occur, i.e., bells whose peak indicates a more robust activation or deactivation for specific values.

\subsection{Summary of Interpretation}

The advantage of this method lies in the fact that it is possible to analyze each input separately in each neuron, thus easily graph each processing function. Then, based on the shape taken by the processing function, we can understand how the input affects the output of a neuron.

The Heaviside is the most interpretable of our models, allowing a direct generation of decision rules.

Sigmoid and $\tanh$-prod cases depend on the parameter $w$. When it is close to $0$, the activation is constant regardless of the input. When $w$ is large enough, the processing function is approximately a piecewise constant function taking only values $0$ and $1$.

In all the other cases, the processing function approximates a linear or bell-shaped function. Even if we can not derive exact decision rules directly from the model, in these cases, we can infer a linear or quadratic relation between input and output.

Each layer aggregates the interpretations of the previous layers. For example, the processing function of a second layer neuron gives a precise activation when its input is greater than a certain threshold, i.e., the bias $b$ of the processing function. The output of the neuron of the first layer must exceed this threshold, and this happens if its processing functions give in output values whose sum exceeds this threshold.

A separate case is the last layer, where the $\alpha$ parameters weigh each of the interpretations generated up to the last layer.

We can interpret a traditional individual neuron as a linear regressor. However, when we add more layers, they cannot be interpreted. Our structure, instead, remains interpretable even as the number of layers increases.

\section{Universality}
\label{sec:universality}

A fundamental property of neural networks is that of universal approximation. Under certain conditions, multilayer feed-forward neural networks can approximate any function in a given function space. In \cite{cybenko} it is proved that a neural network with a hidden layer and using a continuous sigmoidal activation function is dense in $C(I_n)$, i.e., the space of continuous functions in the unit hypercube in $\mathbb{R}^n$.
\cite{hornik89} generalized to the larger class of all sigmoidal functions.

To make the statement of theorems clearer we recall that the structure of a two-layer network with \ian\ neurons and a generic processing function $h$ is
\begin{equation}
    \label{2layer_generica}
    \psi(x) = \sum_{j=1}^N \alpha_j h(w_j(\sum_{i=1}^n h(w_{ij} (x_i - b_{ij}))-b_j))
\end{equation}
where $w_j, w_{ij}, \alpha_j, b_j, b_{ij} \in \mathbb{R}$.


When the processing function is the Heaviside function we proved that the network can approximate any continuous function on $I_n$, Lebesgue measurable functions on $I_n$ and functions in $L^p(\mathbb{R}^n,\mu)$ for $1 \leq p < \infty$, with $\mu$ being a Radon measure. More precisely, the following theorems hold;
we detail the proofs of the theorems in the appendix.

\begin{theorem}
When the processing function is the Heaviside function the finite sums of the form (\ref{2layer_generica}) are dense in $L^p(\mathbb{R}^n, \mu)$ for $1 \leq p < \infty$, with $\mu$ being a Radon measure on $(\mathbb{R}^n,\mathcal{B}( \mathbb{R}^n))$ ($\mathcal{B}$ denote the Borel $\sigma$–algebra). 
\end{theorem}

\begin{theorem}
When the processing function is the Heaviside function the finite sum of the form (\ref{2layer_generica}) are $m$-dense in $M^n$. Where $M^n$ is the set of Lebesgue measurable functions on the $n$-dimensional hypercube $I_n$.
\end{theorem}

\begin{theorem}
Given $g \in C(I_n) $ and given $\epsilon>0 $ there is a sum $\psi(x)$ of the form (\ref{2layer_generica}) with Heaviside as processing function such that $$|\psi(x)-g(x)| < \epsilon \quad \forall x \in I_n.$$
\end{theorem}

When the processing function is the sigmoid function or $\tanh$-prod, we proved that the finite sums of the form (\ref{2layer_generica}) are dense in the space of continuous functions defined on the unit $n$-dimensional hypercube.

\begin{theorem}
When the processing function is a continuous sigmoidal function the finite sums of the form (\ref{2layer_generica})  are dense in $C(I_n)$.
\end{theorem}

\begin{theorem}
Let $\psi(x)$ be the family of networks defined by the equation (\ref{2layer_generica}) when the processing function is given by (\ref{rescaled_prod_tanh}).
This family of functions is dense in $C(I_{n})$.
\end{theorem}

\section{Experiments}
\label{sec:experiments}

\subsection{Datasets}

We selected a collection of datasets from the UCI Machine Learning Repository. We only consider classification models in our experiments. However, it is straightforward to apply \newron architectures to regression problems. The description of the datasets is available at the UCI Machine Learning Repository website or the Kaggle website.

We also used $4$ synthetic datasets of our creation, composed of $1000$ samples with $2$ variables generated as random uniforms between $-1$ and $1$ and an equation dividing the space into $2$ classes. The $4$ equations used are bisector, xor, parabola, and circle.

We give more details about the datasets in the appendix.

\subsection{Methods}

We run a hyperparameter search to optimize the \ian\ neural network structure, i.e., depth and number of neurons per layer, for each dataset. We tested \ian\ with all three different processing functions. In the $\tanh$-prod case, we set $M=2$.

Concerning the training of traditional neural networks, we tested the same structures used for \newron, i.e., the same number of layers and neurons. Finally, we also ran a hyperparameter search to find the best combinations in the case of  Logistic Regression (LR), Decision Trees (DT), and Gradient Boosting Decision Trees (GBDT). We include all the technical details on the methods in the appendix.

\subsection{Results}

\begin{table*}[!ht]
\centering
\begin{tabular}{l||c|c|c||c|c||c|c}
                            & \multicolumn{3}{c||}{\ian\ models}  &     \multicolumn{2}{c||}{Interpretable models}  & \multicolumn{2}{c}{Non-interpretable models} \\
Dataset & Heaviside             & sigmoid               & $\tanh$-prod             &          LR             &            DT           &                GBDT          &            NN         \\ \hline
adult                       & 80.2 ($\pm$0.06)          & \textbf{82.6 ($\pm$0.05)} & 82.3 ($\pm$0.06)          & 76.2 ($\pm$0.07)          & 81.5 ($\pm$0.06)          & \underline{87.5 ($\pm$0.05)}       & 83.1 ($\pm$0.06)        \\
australian                  & 86.5 ($\pm$0.51)          & 87.0 ($\pm$0.5)           & \textbf{88.7 ($\pm$0.4)}  & \textbf{88.7 ($\pm$0.4)}  & 87.0 ($\pm$0.41)          & \underline{90.2 ($\pm$0.47)}       & 88.0 ($\pm$0.4)         \\
b-c-w                       & \textbf{98.9 ($\pm$0.16)} & \textbf{98.9 ($\pm$0.16)} & \textbf{98.9 ($\pm$0.16)} & 97.8 ($\pm$0.23)          & 97.7 ($\pm$0.23)          & 98.3 ($\pm$0.21)             & \underline{98.9 ($\pm$0.17)}  \\
car                         & 95.1 ($\pm$0.2)           & 95.9 ($\pm$0.21)          & \textbf{100.0 ($\pm$0.0)} & 51.4 ($\pm$0.45)          & 98.5 ($\pm$0.11)          & \underline{100.0 ($\pm$0.0)}       & 99.8 ($\pm$0.04)        \\
cleveland                   & \textbf{65.6 ($\pm$1.02)} & 60.1 ($\pm$1.1)           & 62.9 ($\pm$1.13)          & 60.8 ($\pm$1.13)          & 53.6 ($\pm$1.19)          & 61.5 ($\pm$1.01)             & \underline{65.6 ($\pm$1.01)}  \\
crx                         & 86.2 ($\pm$0.51)          & 85.4 ($\pm$0.58)          & 86.5 ($\pm$0.5)           & 84.6 ($\pm$0.45)          & \textbf{88.0 ($\pm$0.42)} & 82.9 ($\pm$0.58)             & 87.7 ($\pm$0.44)        \\
diabetes                    & 73.3 ($\pm$0.56)          & 72.7 ($\pm$0.68)          & \textbf{76.1 ($\pm$0.61)} & 75.6 ($\pm$0.6)           & 74.1 ($\pm$0.63)          & 75.1 ($\pm$0.64)             & 74.2 ($\pm$0.65)        \\
german                      & \textbf{78.2 ($\pm$0.53)} & 77.0 ($\pm$0.53)          & 75.5 ($\pm$0.52)          & 75.1 ($\pm$0.52)          & 68.3 ($\pm$0.57)          & 76.6 ($\pm$0.55)             & 76.7 ($\pm$0.54)        \\
glass                       & 77.0 ($\pm$1.17)          & 81.6 ($\pm$1.04)          & \textbf{85.6 ($\pm$1.02)} & 72.1 ($\pm$1.08)          & 72.7 ($\pm$1.19)          & \underline{87.3 ($\pm$0.9)}        & 82.5 ($\pm$0.91)        \\
haberman                    & 76.9 ($\pm$0.94)          & 76.1 ($\pm$0.92)          & \textbf{77.2 ($\pm$0.88)} & 73.0 ($\pm$1.05)          & 64.4 ($\pm$1.08)          & 72.5 ($\pm$1.09)             & 76.1 ($\pm$0.92)        \\
heart                       & \textbf{88.7 ($\pm$0.67)} & 86.3 ($\pm$0.85)          & 82.7 ($\pm$0.8)           & 82.4 ($\pm$0.95)          & 81.4 ($\pm$1.02)          & 81.7 ($\pm$0.98)             & 82.9 ($\pm$0.95)        \\
hepatitis                   & 84.7 ($\pm$1.26)          & \textbf{85.1 ($\pm$1.23)} & 82.5 ($\pm$1.16)          & 79.1 ($\pm$1.45)          & 79.1 ($\pm$1.33)          & 81.7 ($\pm$1.32)             & 82.4 ($\pm$1.13)        \\
image                       & 93.0 ($\pm$0.11)          & 94.0 ($\pm$0.1)           & \textbf{94.4 ($\pm$0.09)} & 90.4 ($\pm$0.12)          & 90.6 ($\pm$0.12)          & \underline{95.8 ($\pm$0.08)}       & 92.6 ($\pm$0.11)        \\
ionosphere                  & 94.4 ($\pm$0.48)          & \textbf{96.7 ($\pm$0.34)} & 96.5 ($\pm$0.37)          & 92.0 ($\pm$0.51)          & 94.5 ($\pm$0.45)          & 95.4 ($\pm$0.37)             & \underline{96.7 ($\pm$0.34)}  \\
iris                        & \textbf{100.0 ($\pm$0.0)} & \textbf{100.0 ($\pm$0.0)} & \textbf{100.0 ($\pm$0.0)} & \textbf{100.0 ($\pm$0.0)} & 97.3 ($\pm$0.52)          & 97.3 ($\pm$0.52)             & \underline{100.0 ($\pm$0.0)}  \\
monks-1                     & 94.4 ($\pm$0.21)          & \textbf{100.0 ($\pm$0.0)} & \textbf{100.0 ($\pm$0.0)} & 66.0 ($\pm$0.46)          & 90.6 ($\pm$0.27)          & \underline{100.0 ($\pm$0.0)}       & \underline{100.0 ($\pm$0.0)}  \\
monks-2                     & \textbf{100.0 ($\pm$0.0)} & \textbf{100.0 ($\pm$0.0)} & \textbf{100.0 ($\pm$0.0)} & 54.5 ($\pm$0.45)          & 82.7 ($\pm$0.33)          & 94.2 ($\pm$0.21)             & 87.6 ($\pm$0.27)        \\
monks-3                     & 97.1 ($\pm$0.15)          & 97.1 ($\pm$0.15)          & 97.1 ($\pm$0.15)          & 81.2 ($\pm$0.31)          & \textbf{97.2 ($\pm$0.16)} & 96.2 ($\pm$0.16)             & 90.3 ($\pm$0.25)        \\
sonar                       & 93.3 ($\pm$0.74)          & \textbf{96.8 ($\pm$0.48)} & 95.2 ($\pm$0.53)          & 89.5 ($\pm$0.75)          & 83.4 ($\pm$0.98)          & 88.1 ($\pm$0.9)              & 89.4 ($\pm$0.87)        \\
bisector                    & 98.9 ($\pm$0.13)          & 99.3 ($\pm$0.09)          & 99.3 ($\pm$0.09)          & \textbf{100.0 ($\pm$0.0)} & 97.7 ($\pm$0.18)          & 98.3 ($\pm$0.16)             & \underline{100.0 ($\pm$0.0)}  \\
xor                         & \textbf{100.0 ($\pm$0.0)} & \textbf{100.0 ($\pm$0.0)} & 99.2 ($\pm$0.11)          & 53.2 ($\pm$0.65)          & 99.2 ($\pm$0.12)          & \underline{100.0 ($\pm$0.0)}       & \underline{100.0 ($\pm$0.0)}  \\
parabola                    & 98.8 ($\pm$0.15)          & \textbf{100.0 ($\pm$0.0)} & 99.6 ($\pm$0.07)          & 77.8 ($\pm$0.52)          & 97.6 ($\pm$0.18)          & 97.7 ($\pm$0.17)             & \underline{100.0 ($\pm$0.0)}  \\
circle                      & 96.8 ($\pm$0.22)          & 99.3 ($\pm$0.1)           & \textbf{99.6 ($\pm$0.07)} & 52.4 ($\pm$0.67)          & 98.8 ($\pm$0.13)          & 97.6 ($\pm$0.2)              & 99.2 ($\pm$0.11)       
\end{tabular}
\caption{Datasets accuracy ($\pm$ $95^{th}$ percentile standard error) results of the best performing model. In \textbf{bold} we indicate the best performing model amongst the interpretable ones. If GBDT or NN exceeds this accuracy, the corresponding result is \underline{underlined}.}
\label{tab:results}
\end{table*}

Table \ref{tab:results} presents on each row the datasets used while on the columns the various models. Each cell contains the $95\%$ confidence interval for the accuracy of the model that obtains the best performance.

Results obtained with the new \ian\ neurons are better than those obtained by DTs and LRs (interpretable) models. Moreover, \ian's results are on par, sometimes better than, results of traditional NNs and GBDT classifiers. These last two methods, though, are not transparent.

Amongst the Heaviside, sigmoid, and $\tanh$-prod cases, we can see that the first one obtains the worst results. The reason may be that it is more challenging to train, despite being the most interpretable among the three cases. $\tanh$-prod instead performs slightly better than sigmoid, being more flexible. Sigmoid, being more straightforward to interpret than $\tanh$-prod, could be a good choice at the expense of a slight decrease in accuracy that remains, however, similar to that of a traditional neural network.

\subsection{Circle dataset example}
In order to first validate our ideas, we show what we obtained by applying a single neuron using multiplication of $2$ $\tanh$ in the case of our custom dataset circle.

\begin{figure}[ht]
\centering
    \includegraphics[width=0.9\columnwidth]{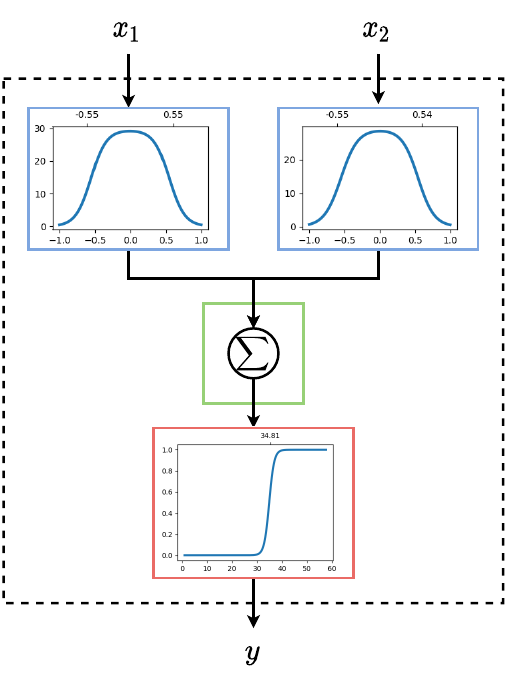}
    \caption{$\tanh$-prod Neural Network trained on the circle dataset. The figure follows the color convention used for \newron\ in Figure \ref{fig:newron}. $x_1$ and $x_2$ are the inputs of the network and $y$ is the output. The processing and activation functions are plotted with input on the $x$-axis and output on the $y$-axis. Coordinates of the inflection points are indicated above the plots.}
    \label{custom_circumference}
\end{figure}

In Figure \ref{custom_circumference} we can see how the multiplication of $\tanh$ has converged to two bells centred in $0$, while $\alpha_1$ and $\alpha_2$ have gone to $30$. According to the \ian interpretation method, values below $30$ correspond to an activation function output of $0$, while it is $1$ for values above $38$. In the middle range, the prediction is more uncertain. Combining this data with the previous prediction, we can conclude that we need the sum of the two values output by the two processing functions to be greater than $38$ to have a prediction of class $1$. Therefore, if one of the two inputs is $0$ (output $30)$, it is enough for the other to be between $-0.65$ and $0.65$ (output greater than $8$). Otherwise, we may need an output of at least $19$ from both outputs, corresponding to input values between $-0.5$ and $0.5$, i.e., the area covered by the circle. We show more examples in the appendix.

\subsection{Current limitations}
The extraction of proper rules from the network can be harrowing; in the Heaviside case, they might be too long in the sigmoid and $\tanh$-prod cases because their simplicity depends on the final value parameters. Nevertheless, methods of regularization during training or additional Rule Extraction methods may help to simplify interpretability. We defer the study of regularization to future works.

Also, we have not compared \newron\ against state-of-the-art Deep Learning models for tabular data, as our main goal was to show that our formulation was more suitable than traditional neurons compared to ``traditional'' interpretable models. Comparisons with more advanced solutions for tabular data will be the subject of future work. 

\section{Conclusions and Future Work}
\label{sec:conclusion}
We have introduced the concept of a generalized neuron and proposed three different specializations, along with the corresponding method to interpret the behavior of the network. Also, in cases where from the network we cannot devise exact rules (e.g., in the sigmoid and $\tanh$-prod cases), the structure of the neuron and the parameters allow the visualization of its behavior. Indeed, for every input, we apply the nonlinearity operation before the aggregation reducing it to a one-dimensional space allowing the analysis of each input separately.
Through universal approximation theorems, we have proved that the new structure retains the same expressive power as a standard neural network. In future studies we will investigate more in detail the expressiveness of \ian\ based models with respect to the number of layers or neurons in arbitrarily deep but width-limited networks and arbitrarily wide but depth-limited networks.
Experiments conducted on both real and synthetic datasets illustrate how our framework can outperform traditional interpretable models, Decision Trees, and Logistic Regression, and achieve similar or superior performance to standard neural networks.
In the future, we will investigate the influence of hyper-parameters (network depth, number of neurons, processing functions) and initialization on the model quality. Also, we will refine the analysis of the $\tanh$-prod case as the number of $\tanh$ increases. In addition, we will investigate \ian\ with additional processing functions, such as ReLU and SeLU. Finally, we will extend this method to other neural models, such as Recurrent, Convolutional and Graph Neural Networks.

\section{Acknowledgements}
This research was supported by the Italian Ministry of Education, University and Research (MIUR) under the grant ``Dipartimenti di eccellenza 2018--2022'' of the Department of Computer Science and the Department of Computer Engineering at Sapienza University of Rome. Partially supported by the ERC Advanced Grant 788893 AMDROMA ``Algorithmic and Mechanism Design Research in Online Markets'', the EC H2020RIA project ``SoBigData++'' (871042), and the MIUR PRIN project ALGADIMAR ``Algorithms, Games, and Digital Markets''. All content represents the opinion of the authors, which is not necessarily shared or endorsed by their respective employers and/or sponsors.

\bibliography{newron}

\clearpage
\begin{center}
\textbf{\LARGE Supplementary Materials}
\end{center}

\appendix

\section{Universality Theorems}

This is the appendix to the Universality section in the main article. In this section, we shall prove the mathematical results concerning the universal approximation properties of our \ian\ model. 
In particular, we restrict ourselves to some specific cases. We consider the cases where the processing function is the Heaviside function, a continuous sigmoidal function ,or the rescaled product of hyperbolic tangents.

\subsection{Heaviside \ian}

\begin{theorem}
\label{teo:HeavisideLp}
The finite sums of the form
\begin{equation}
    \label{reteHeav}
\psi(x)=   \sum_{j=1}^N \alpha_j H(w_j\sum_{i=1}^n H(w_{ij} (x_i - b_{ij}))- b_j) \end{equation}
with $N$ $ \in \mathbb{N}$ and  $w_{ij}, w_{j}, \alpha_j, b_{ij}, b_j  \in \mathbb{R}$ 
are dense in $L^p(\mathbb{R}^n, \mu)$ for $1 \leq p < \infty$, with $\mu$ a Radon measure on $
(\mathbb{R}^n,\mathcal{B}(\mathbb{R}^n))$ ($\mathcal{B}$ denote the Borel $\sigma$–algebra).
\end{theorem}

In other words given, $g \in L^p(\mathbb{R}^n,\mu )$ and $\epsilon >0 $ there is a sum $\psi(x)$ of the above form for which $$||\psi-g||^p_{p} = \int_{\mathbb{R}^n} |\psi(x)-g(x)|^p d\mu(x)  < \epsilon.$$

To prove that a neural network defined as in equation (\ref{reteHeav}) is a universal approximator in $L^p, $ for $ 1 \leq p < \infty $ we exploit that step functions are dense in $L^p$ and that our network can generate step functions. 
\begin{proposition}
\label{prop:step_function_dense_Lp}
Let $\mathcal{R}$ be the set of the rectangles in $\mathbb{R}^n$ of the form
$$ R= \prod_{k=1}^n [a_k, b_k) \quad a_k, b_k \in \mathbb{R} , \; a_k < b_k 
$$
We denote by $\mathcal{F} $ the vector space on $\mathbb{R}$ generated by $ \mathds{1}_R, \: R \in \mathcal{R} $ i.e.
\begin{equation}
\label{set:F}
\mathcal{F} = \Big\{ \sum_{i=1}^m \alpha_{i} \mathds{1}_{R_i} \: \Big| \: m \in \mathbb{N}, \alpha_i \in \mathbb{R}, R_i \in \mathcal{R}\Big\} 
\end{equation}
$\mathcal{F}$ is dense in $L^p(\mathbb{R}^n, \mu)$ for $1 \leq p < \infty$, with $\mu$ a Radon measure on $(\mathbb{R}^n,\mathcal{B}(\mathbb{R}^n))$.
\end{proposition}
\begin{proof}
See chapter 3, $L^p$ Spaces , in \cite{cannarsa}.
\end{proof}

\begin{lemma}
\label{lemma:Heaviside_stepfunction}
Given $\rho(x) \in \mathcal{F}$, with $\mathcal{F}$ defined as in equation (\ref{set:F}), there exists a finite sum $\psi(x) $ of the form (\ref{reteHeav}) such that $\rho(x)= \psi(x) \; \forall x \in \mathbb{R}^n$.
\end{lemma}
\begin{proof}

To prove that a neural network described as in equation (\ref{reteHeav}) can generate step functions we proceed in two steps. First, we show how we can obtain the indicator functions of orthants from the first layer of the network. Then we show how, starting from these, we can obtain the step functions.

An \textit{orthant} is the analogue in $n$-dimensional Euclidean space of a quadrant in $\mathbb{R}^2$ or an octant in $\mathbb{R}^3$.
We denote by\textit{ translated orthant} an orthant with origin in a point different from the origin of the Euclidean space $O$.
Let $A$ be a point in the $n$-dimensional Euclidean space, and let us consider the intersection of $n$ mutually orthogonal half-spaces intersecting in $A$. By independent selections of half-space signs with respect to $A$ (i.e. to the right or left of $A$) $2^n$ orthants are formed.

Now we shall see how to obtain translated orthant with origin in in a point $A$ of coordinates $(a_1,a_2,...,a_n)$ from the first layer of the network i.e. $ \sum_{i=1}^n H(w_{i} (x_i - b_{i})) $.

For this purpose we can take $ w_i=1  \quad \forall i \in \{1,...,n\}$.

The output of $ \sum_{i=1}^n H(x_i - b_{i}) \in \{0,...,n\}$ and depends on how many of the $n$ Heaviside functions are activated.
We obtain the translated orthant with origin in $A$ by choosing $ b_i= a_i \quad \forall i \in \{1,...,n\}$. In fact,

\begin{equation*}
H(x_i - a_{i})= 
\begin{cases}
0 \: & \text{ if } x_i< a_i \\ 
1 \: & \text{ if }  x_i \geq a_i.  \\
\end{cases}
\end{equation*}

The $i$-th Heaviside is active in the half-space $ x_i \geq a_i $ delimited  by the hyperplane $ x_i= a_i $ that is orthogonal to the $i$-th axis. 
Therefore, the Euclidian space $\mathbb{R}^n$ is divided in $2^n$ regions according to which value the function $\sum_{i=1}^n H( x_i - a_{i})$ takes in each region. See Figure \ref{fig_Orthant_A} for an example in $ \mathbb{R}^2$.

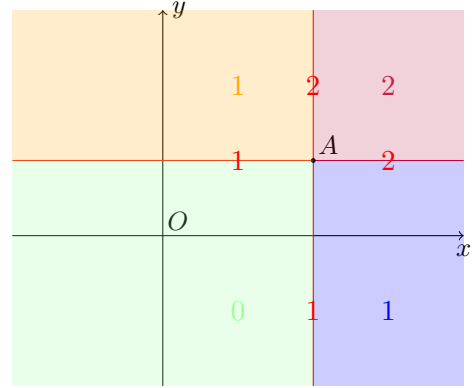
\begin{figure}[ht]
\begin{tikzpicture}[]

\draw[->] (-3,0)--(3,0) node[below]{$x$};
\draw[->] (-1,-2)--(-1,3) node[right]{$y$};
\draw[red] (1,-2)--(1,3) ;
\draw[red] (-3,1)--(3,1);
\draw (-0.8,0.2) node{$O$};
\draw (1.2,1.2) node{$A$};
\filldraw[black] (1,1) circle (0.7pt);
\draw  [draw opacity=0][fill={brightmaroon}  ,fill opacity=0.2 ] (1,1) -- (3,1) -- (3,3) -- (1,3) -- cycle;
\draw  [draw opacity=0][fill={blue}  ,fill opacity=0.2 ] (1,1) -- (3,1) -- (3,-2) -- (1,-2) -- cycle;
\draw  [draw opacity=0][fill={chromeyellow}  ,fill opacity=0.2 ] (1,1) -- (-3,1) -- (-3,3) -- (1,3) -- cycle;
\draw  [draw opacity=0][fill={mintgreen}, fill opacity=0.2 ] (-3,1) -- (1,1) -- (1,-2) -- (-3,-2) -- cycle;
\draw[brightmaroon] (2,2) node{\large\textbf{$2$}};
\draw[blue] (2,-1) node{\large\textbf{$1$}};
\draw[chromeyellow] (0,2) node{\large\textbf{$1$}};
\draw[mintgreen] (0,-1) node{\large\textbf{$0$}};
\draw[red] (1,2) node{\large\textbf{$2$}};
\draw[red] (1,-1) node{\large\textbf{$1$}};
\draw[red] (2,1) node{\large\textbf{$2$}};
\draw[red] (0,1) node{\large\textbf{$1$}};

\end{tikzpicture}
\caption{Partition of $ \mathbb{R}^2, $ according to output of the function $ H(x_1-a_1) + H(x_2- a_2) $.  $A$ is a point of coordinates $(a_1,a_2)$.}
\label{fig_Orthant_A}
\end{figure}

There is only one region in which the output of the sum is $n$, which corresponds to the orthant in which the condition $x_i \geq a_i \: \forall i= 1,...,n $ holds. We denote it as \textit{positive othant} (the red colored orthant in the example shown in Figure \ref{fig_Orthant_A}).

Going back to equation (\ref{reteHeav}), let us now consider the Heaviside function applied after the sum. As before, we can choose $w_j=1$. If we take $b_j > n-1$, the value of the output is $0$ for each of the $2^n$ orthants except for the positive orthant. This way, we get the indicator function of the positive orthant.

The indicator function of a rectangle in $\mathcal{R}$ can be obtained as a linear combination of the indicator function of the positive orthants centered in the vertices of the rectangle. See Figure \ref{SquareR2} for an example of the procedure in $\mathbb{R}^2$.

In general, the procedure involves considering a linear combination of indicator functions of positive orthants centered in the vertices of the rectangle in such a way that opposite values are assigned to the orthants corresponding to adjacent vertices.

For example, suppose we want to obtain the indicator function of the right-closed left-open square $[0,1)^2$ in $\mathbb{R}^2$ (see the illustration in Figure \ref{SquareR2}). Denoting by 
$\mathds{1}_{(x_P,y_P)\llcorner }$ the indicator function of the positive orthant centered in the point of coordinates $(x_P,y_P)$, we can write:
$$\mathds{1}_{[0,1)^2}= \mathds{1}_{(0,0) \llcorner } - \mathds{1}_{(1,0)\llcorner } - \mathds{1}_{(0,1)\llcorner } + \mathds{1}_{(1,1)\llcorner }.$$

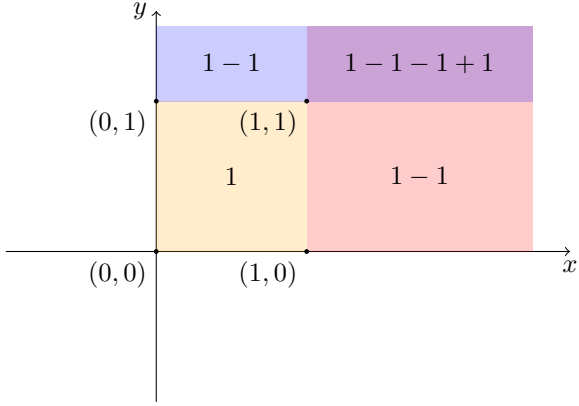
\begin{figure}[ht]
\begin{tikzpicture}[]

\draw[->] (-3,0)--(4.5,0) node[below]{$x$};
\draw[->] (-1,-2)--(-1,3.2) node[left]{$y$};
\draw  [draw opacity=0 ][fill={chromeyellow}  ,fill opacity=0.2 ](1,0) -- (1,2) -- (-1,2) --(-1,0)  -- cycle ;
\draw  [draw opacity=0 ][fill={red}  ,fill opacity=0.2 ](1,0) -- (1,3) -- (4,3) --(4,0)  -- cycle ;
\draw  [draw opacity=0 ][fill={blue}  ,fill opacity=0.2 ](-1,2) -- (4,2) -- (4,3) --(-1,3)  -- cycle ;
\draw (-1,0)node[below left ]{$ (0,0) $};
\filldraw[black] (-1,0) circle (0.7pt);
\draw (1,0)node[below left]{$ (1,0) $};
\filldraw[black] (1,0) circle (0.7pt);
\draw (-1,2)node[below left]{$ (0,1)$};
\filldraw[black] (-1,2) circle (0.7pt);
\draw (1,2)node[below left]{$ (1,1)$};
\filldraw[black] (1,2) circle (0.7pt);
\draw (0,1)node{$1$};
\draw (2.5,1)node{$1 -1$};
\draw (0,2.5)node{$1 -1$};
\draw (2.5,2.5)node{$1 -1 -1 +1$};
\end{tikzpicture}
\caption{How to obtain the indicator function on the square $[0,1)^2$ from the linear combination of four indicator functions of positive orthants  centered in the vertices of $[0,1)^2$. 
$\mathds{1}_{[0,1)^2}= \mathds{1}_{(0,0) \llcorner } - \mathds{1}_{(1,0)\llcorner } - \mathds{1}_{(0,1)\llcorner } + \mathds{1}_{(1,1)\llcorner }.$
The numbers in the orthants shows the sum of the indicator functions that are active in that orthant.
For instance if $x=(x_1,x_2)$ belongs to the blue part of the plane, i.e. it is true that  $0<x_1<1$ and $x_2>1$, we have that
$\mathds{1}_{(0,0) \llcorner}(x) - \mathds{1}_{(1,0)\llcorner}(x) - \mathds{1}_{(0,1)\llcorner}(x) + \mathds{1}_{(1,1)\llcorner}(x) = 1-0-1+0=1-1.$ } 
\label{SquareR2}
\end{figure} 

Now suppose we want the linear combination of the indicator functions of $K$ rectangles with coefficents $ \alpha_1,... \alpha_{K}$. With suitably chosen coefficients the indicator function of a rectangle can be written as $$\sum_{l=1}^{2^n} (-1)^l H(w_{jl}\sum_{i=1}^n H(w_{ij} (x_i - b_{ij}))- b_{jl})$$
that replacing $H(w_{jl}\sum_{i=1}^n H(w_{ij} (x_i - b_{ij}))- b_{jl})$ by $H_l$, to abbreviate the notation becomes
$$\sum_{l=1}^{2^n} (-1)^l H_l .$$ 

The linear combination of the indicator functions of $K$ rectangles with coefficents $ \alpha_1,... \alpha_{K}$ can be derived as 
\begin{equation}
\label{linear_combination}\sum_ {k=1}^{K}   \alpha_k \sum_{l=1}^{2^n}  (-1)^l H_{lk}. \end{equation}
The summation (\ref{linear_combination}) can be written as a single sum, defining a sequence $
\beta_j= (-1)^{j}\alpha_{m} \text{ with } m=\lceil{{\frac{j}{2^n}} }\rceil \text{ for } j= 1,...,2^nK$. Thus (\ref{linear_combination}) becomes  $$ \sum_ {j=1}^{N= 2^nK} \beta_j H_j$$
that is an equation of form (\ref{reteHeav}).
We have therefore shown that for every step function $\rho$ in $\mathcal{F}$ there are $N \in \mathbb{N}$ and $\alpha_j,$ $w_{ij},$ $b_{ij},$ $b_{j}, w_j \in \mathbb{R}$ such that the sum in equation (\ref{reteHeav}) is equal to $\rho$.
\end{proof}

\begin{proof}[Proof of Theorem \ref{teo:HeavisideLp}]
The theorem follows immediately from Lemma \ref{lemma:Heaviside_stepfunction} and Proposition \ref{prop:step_function_dense_Lp}. 
\end{proof}


\begin{remark}
\label{remark:rectangles-1}
In  Lemma \ref{lemma:Heaviside_stepfunction} we proved that a network defined as in equation (\ref{reteHeav}) can represent functions belonging to set $\mathcal{F}$ defined as in equation (\ref{set:F}). 
Note that if the input is bounded, we can obtain indicator functions of other kinds of sets.
For example, suppose $x \in [0,1]^n$. If we choose $w_{ij}=1$ and $b_{ij}< 0 \; \forall i,j$ and if we choose the weights of the second layer so that they don't operate any transformation, we can obtain the indicator function of $[0,1]^n$.
By a suitable choice of parameters, (\ref{reteHeav}) may also become the indicator functions of any hyperplane $x_i=0$ or $x_i=1$ for $i \in \{1,..,n \}$. Furthermore we can obtain any rectangle of dimension $n-1$ that belongs to an hyperplane of the form $ x_i=1 $ or $x_i=0$.
\end{remark}

We have proven in Lemma \ref{lemma:Heaviside_stepfunction} that a network formulated as in equation (\ref{reteHeav}) can represent step functions. By this property and by Proposition \ref{th6} we shall show that it can approximate Lebesgue measurable functions on any finite space, for example the unit $n$-dimensional cube $[0,1]^n$.

We denote by $I_n$ the closed $n$-dimensional cube $[0,1]^n$. 
We denote by $M^n$ the set of measurable functions with respect to Lebesgue measure $m$, on $I_n$, with the metric $d_m$ defined as follows.
Let be $f,g \in M^n$, 
$$d_m(f, g) = \inf \{\epsilon > 0: m \{ x: |f(x) - g(x)| > \epsilon \} < \epsilon \}$$

We remark that $d_m$-convergence is equivalent to convergence in measure (see Lemma 2.1 in \cite{hornik89}).
\begin{theorem}
\label{teo:HeavisideMeas}

The finite sums of the form (\ref{reteHeav}) with $N$ $ \in \mathbb{N}$ and $w_{ij}, w_{j}, \alpha_j, b_{ij}, b_j \in \mathbb{R}$ are $d_m$-dense in $M^n$. $M^n$ is the set of Lebesgue measurable functions on $I_n$ .
\end{theorem}
This means that, given $g$ measurable with respect to the Lebesgue measure $m$ on $I_n$, and given an $\epsilon>0$, there is a sum $\psi $ of the form (\ref{reteHeav}) such that $d_m(\psi,g) < \epsilon.$

\begin{proposition}\label{th6}
Suppose $f$ is measurable on $\mathbb{R}^n$. Then there exists a sequence of step functions $\{\rho_k\}_{k=1}^\infty$ that converges pointwise to $f(x)$ for almost every $x$.
\end{proposition}

\begin{proof}
See Theorem 4.3 p. 32 in \cite{stein}.
\end{proof}

\begin{proof}[Proof of Theorem \ref{teo:HeavisideMeas}]

Given any measurable function, by Proposition \ref{th6} there exists a sequence of step functions that converge to it pointwise. By Lemma \ref{lemma:Heaviside_stepfunction} we have that equation (\ref{reteHeav}) can generate step functions. 
Now $m(I_n)=1$ and for a finite measure space pointwise convergence implies convergence in measure, this concludes the prof. 
\end{proof}

\begin{remark}
Notice that for Theorem \ref{teo:HeavisideMeas} we don't need the fact that $I_n$, is a closed set. For this theorem in fact it is sufficient that it is a bounded set (so that its Lebesgue measure is finite).
The compactness of $I_n$ will be necessary for the next theorem.

\end{remark} 



\begin{theorem}
\label{teo:Heaviside_continuous}
Given $g \in C(I_n) $ and given $\epsilon>0 $ there is a sum $\psi(x)$ of the form (\ref{reteHeav}) such that $$|\psi(x)-g(x)| < \epsilon \quad \forall x \in I_n.$$
\end{theorem}

\begin{proof}
Let $g$ be a continuous function from $I_n$ to $\mathbb{R}$, by the compactness of $I_n$ follows that $g$ is also uniformly continuous (see Theorem 4.19  p. 91 in \cite{rudin_principles}). In other words, for any $\epsilon> 0$, there exists $\delta >0$ such that for every $x, x' \in [0,1]^n$ such that $||x-x' ||_\infty < \delta $ it is true that $ |g(x) -g(x')| < \epsilon $.
To prove the statement of Theorem \ref{teo:Heaviside_continuous}, let $\epsilon>0$ be given, and let $\delta>0$ be chosen according to the definition of uniform continuity.

As we have already seen in Lemma \ref{lemma:Heaviside_stepfunction} the neural network described in (\ref{reteHeav}) can generate step functions with support on right-open left-closed $n$-dimensional rectangles and on $(n-1)$-dimensional rectangles that belongs to an hyperplane of equation $x_i=0$ or $x_i=1 $ for some $i \in \{1,...,n \}$ as seen in Remark \ref{remark:rectangles-1}.
There exists a partition of $[0,1]^n$, $(R_1,...,R_N)$, consisting of right-open left-closed $n$-dimensional rectangles and of $(n-1)$-dimensional rectangles that belongs to an hyperplane of equation $x_i=0$ or $x_i=1 $ for some $i \in \{1,...,n \}$, such that all side lengths are no greater than $\delta$. An example of a set of rectangles with this property is the set of right-open left-closed cubes of side length $\frac{1}{\tilde{m}}, \tilde{m}> \lceil \frac{1}{\delta} \rceil$ with the $(n-1)$-dimensional rectangles with the same side length which we need to cover all the boundary of $[0,1]^n$ not covered by the right-open left-closed rectangles.

Suppose that for all $j \in \{1,...,N\}$ we choose $x_j \in R_j$, and we set $\alpha_j=g(x_j)$.
If $x \in [0,1]^n$ there is $j$ so that $x \in R_j$,
hence $x$ satisfies
$||x-x_j||_\infty \leq \delta$, 
and consequentially 
$ |g(x) - g(x_j)| \leq \epsilon. $
Therefore the step function
$h= \sum_{j=1}^N \alpha_j \mathds{1}_{R_j}$
satisfies

$$\sup_{x \in I_n} |h(x)-g(x)| =$$
$$=\sup_{j \in \{1,...,N\}} \sup_{x \in R_j} |h(x)-g(x)|= $$
$$=\sup_{j \in \{1,...,N\}} \sup_{x \in R_j} |\alpha_j-g(x)| \leq \epsilon $$

\end{proof}

\subsection{Sigmoid \ian}


\begin{definition}
A function
$ \sigma : \mathbb{R} \to [0,1]$ is called sigmoidal if $$ \lim_{x \to - \infty} \sigma(x) =0, \quad \lim_{x \to + \infty} \sigma(x) =1 $$
\end{definition}

\begin{theorem}
\label{teo:sigmoide}
Let $\sigma$ be a continuos sigmoidal function. Then the finite sums of the form:

\begin{equation}
    \label{eq:reteSigmoide1_appendix}
    \psi(x)=\sum_{j=1}^N \alpha_j \sigma(w_j(\sum_{i=1}^n \sigma(w_{ij}( x_i -b_{ij}))-b_j))
\end{equation}
with $w_{ij}, \alpha_j, b_{ij}, b_j, w_j \in \mathbb{R}$ and $N \in \mathbb{N}$ are dense in $C(I_n)$.
\end{theorem}

In other words, given a $g \in C(I_n) $ and given $\epsilon>0$ there is a sum $\psi(x)$ of the form (\ref{eq:reteSigmoide1_appendix}) such that 
$$|\psi(x) -g(x)|<\epsilon \quad \forall x \in I_n. $$

\begin{proof}
Since $\sigma$ is a continuous function, it follows that the set $U$ of functions of the form (\ref{eq:reteSigmoide1_appendix}) with $\alpha_j,w_{ij},b_{ij}, w_j,b_j \in \mathbb{R}$ and $N \in \mathbb{N}$ is a linear subspace of $C(I_n)$. We claim that the closure of $U$ is all of $C(I_n)$.

Assume that $U$ is not dense in $C(I_n)$, let $S$ be the closure of $U$, $S \neq C(I_n)$.
By the Hahn-Banach theorem (see p. 104 of \cite{rudin_real_complex} ) there is a bounded linear functional on $C(I_n)$, call it $L$, with the property that $L \neq 0 $ but $L(S)=L(U)=0$.

By the Riesz Representation Theorem (see p. 40 of \cite{rudin_real_complex}), the bounded linear functional $L$, is of the form

$$ L(f) = \int_{I_n} f(x) d \mu$$

for some signed regular Borel measures $\mu $ such that $\mu(K) < \infty $ for every compact set $K \subset I_n $ (i.e. $\mu$ is a Radon measure). Hence,

\begin{equation}
\label{integrale_mu}
\int_{I_n} h(x) d \mu =0, \forall h \in U.
\end{equation} 
We shall prove that (\ref{integrale_mu}) implies $\mu=0$, which contradicts the hypothesis $L\neq 0$. 

Using the definition of $U$, equation (\ref{integrale_mu}) can also be written as 
$$ \sum_{j=1}^N \alpha_j \int_{I_n} \sigma(w_j( \sum_{i=1}^n \sigma( w_{ij}( x_i - b_{ij})) - b_j)) d\mu=0,$$
for any choice of $\;\alpha_j, w_{ij}, w_j, b_{ij}, b_j \in \mathbb{R}$ and $N\in \mathbb{N}$.

Note that for any $ w, x, b \in \mathbb{R} $ we have that
the continuous functions 
$$ \sigma_{\lambda} (w(x - b) ) = \sigma(\lambda w( x -b) + \phi )$$
converge pointwise to the unit step function as $\lambda $ goes to infinity, i.e.
$$ \lim_{\lambda \to \infty} \sigma_{\lambda}(w(x-b)) = \gamma(w(x-b))$$
with 
$$
 \gamma(y) = \begin{cases}
  1 & \text{ if } y>0 \\
  \sigma(\phi) & \text{ if } y=0 \\
   0 & \text{ if } y<0
\end{cases} 
$$

By hypothesis is true that for all $\lambda_1, \lambda_2 $ in $\mathbb{R}$
\begin{equation*}
\int_{I_n} \sigma_{\lambda_2}(w_j(\sum_{i=1}^n \sigma_{\lambda_1}(w_{ij} (x_i - b_{ij}))- b_j)) d\mu=0.
\end{equation*}

It follows that for all $\lambda_{2}$:

\begin{equation*}
\lim_{{\lambda_1} \to \infty } 
\int_{I_n} \sigma_{\lambda_2}(w_j(\sum_{i=1}^n \sigma_{\lambda_1}(w_{ij} (x_i - b_{ij}))- b_j)) d\mu=0.
\end{equation*}

Now applying the Dominated Convergence Theorem (see Theorem 11.32 p 321 of \cite{rudin_principles}) and the fact that $\sigma$ is continuous:

\begin{equation*}
    \begin{aligned}
\int_{I_n}  \lim_{{\lambda_1} \to \infty }  \sigma_{\lambda_2}(w_j(\sum_{i=1}^n \sigma_{\lambda_1}(w_{ij}( x_i - b_{ij}))- b_j)) d\mu=\\
\int_{I_n}  \sigma_{  \lambda_2}(w_j(\sum_{i=1}^n  \gamma(w_{ij}( x_i - b_{ij}))- b_j)) d\mu.
   \end{aligned}
\end{equation*}
 
Again, by Dominated Convergence Theorem we have:

\begin{equation*}
    \begin{aligned}
\lim_{\lambda_2 \to \infty} &\int_{I_n}  \sigma_{\lambda_2}(w_j(\sum_{i=1}^n  \gamma(w_{ij}( x_i - b_{ij}))- b_j)) d\mu = \\
   &\int_{I_n}  \gamma (w_j(\sum_{i=1}^n \gamma(w_{ij} (x_i - b_{ij}))- b_j)) ) d \mu.  
    \end{aligned}
\end{equation*}

Hence we have obtained that  $ \forall \alpha_j,w_{ij}, b_{ij},  w_j, b_{j} \in \mathbb{R}$ and $ \forall N \in \mathbb{N}$
$$\int_{I_n} \sum_{j=1}^N \alpha_j \gamma (w_j(\sum_{i=1}^n \gamma(w_{ij} (x_i - b_{ij}))- b_j)) d \mu =0. $$

The function $\gamma$ is very similar to the Heaviside function $H$, the only difference is that $H(0)=1$ while $\gamma(0)=\sigma(\phi)$.
Let $R_i$ denote an open rectangle, $\partial_a R_i$ its left boundary (i.e. the boundary of a left-closed right-open rectangle) and $\partial_b R_i $ its right boundary (i.e. the boundary of a right-closed left-open rectangle).
Repeating the construction seen in Lemma \ref{lemma:Heaviside_stepfunction} to obtain rectangles,  with the difference that here $\gamma$ takes value $\sigma(\phi)$ on the boundaries, we get that 
$$ \sigma(\phi)\mu (\partial _a R_i)+ (1- \sigma(\phi))\mu (\partial _{b}R_i) + \mu (R_i) =0$$
for every open rectangle $R_i.$ 
Taking $\phi \to \infty$, implies
$$ \mu (\partial_{a }R_i) + \mu (R_i) =0 \quad \forall \textrm{ open rectangle } R_i.$$

Every open subset $A$ of $I_n$, can be written as a countable union of disjoint partly open cubes (see Theorem 1.11 p.8 of \cite{zygmund}).
Thus, from the fact that the measure is $\sigma$-additive we have that for every open subset $A$ of $I_n$, $ \mu(A)=0.$
Furthermore $\mu (I_n) =0$. To obtain $I_n$ from $$ \sum_{j=1}^N \alpha_j \gamma (w_j(\sum_{i=1}^n \gamma(w_{ij} (x_i - b_{ij}))- b_j))$$ it is sufficient to choose the parameters so that $w_{ij} (x_i - b_{ij}) >0 \; \forall x_i \in [0,1]$ and so that $w_j, b_j$ maintains the condition on the input. 

Hence, $\mu (A^C) = \mu(I_n ) - \mu (A)=0.$
It follows that for all compact set $K$ of $I_n$,
 $ \mu (K)=0$.

From the regularity of the measure, it follows that $\mu $ is the null measure.



\end{proof} 

 \subsection{$\tanh$-prod \ian}

\begin{theorem}
\label{teo:tanh}
The finite sums of the form 
\begin{equation}
\begin{aligned}
\label{eq:rete_tanh}
        \psi(x)=\sum_{j=1}^N\frac{\alpha_j}{2}\left[\prod_{l=1}^{M_j}\tanh(w_{jl}(z_j(x)-b_{jl}))+1\right] \\
    z_j(x)=\sum_{i=1}^n\frac{1}{2}\left[\prod_{k=1}^{m_i}\tanh(w_{ijk}(x_i-b_{ijk}))+1\right]
    \end{aligned}
\end{equation}
with $w_{jl}, w_{ijk}, \alpha_j, b_{jl}, b_{ijk} \in \mathbb{R}$ and $M_j,N,m_i \in \mathbb{N}$, 
are dense in $C(I_n)$.

In other words given $g \in C(I_n) $ and given $\epsilon>0 $ there is a sum $\psi(x)$ defined as above
such that 
$$|\psi(x) -g(x) | < \epsilon \quad \forall x \in I_n.$$

\end{theorem}


Since $\tanh$ is a continuous function, it follows that the family of functions defined by equation (\ref{eq:rete_tanh}) is a linear subspace of $C(I_n)$. To prove that it is dense in $C(I_n)$ we will use the same argument we used for the continuous sigmoidal functions.

This is, called $U$ the set of functions of the form (\ref{eq:rete_tanh}), we assume that $U$ is not dense in $C(I_n)$.
Thus, by the Hahn-Banach theorem there exists a not null bounded linear functional on $C(I_n)$ with the property that it is zero on the closure of $U$.
By the Riesz Representation Theorem, the bounded linear functional can be represented by a Radon measures. Then using the definition of $U$ we will see that this measure must be the zero measure, hence the functional associated with it is null contradicting the hypothesis.

We define
\begin{equation}
    \label{eq:1layer_tanh}
h_{\lambda}(x)=\frac{1}{2}\left[\prod_{k=1}^{m}\tanh(\lambda(w_{k}(x-b_{k}))+ \phi)+1\right]. 
\end{equation}
To proceed with the proof as in the case of the proof for continuous sigmoidal functions, we need only to understand to what converges the function
\begin{equation}
\label{eq:limite_rete_tanh}
        \psi_{\lambda_2, \lambda_1}(x)=\sum_{j=1}^N\frac{\alpha_j}{2}h_{j\lambda_2}\left( \sum_{i=1}^n h_{i\lambda_1}(x)\right)
\end{equation}

when $\lambda_1$ and $\lambda_2$ tend to infinity, and $h_{i\lambda}$ indicates the processing function related to input $i$.

Once we have shown that for some choice of the parameters they converge pointwise to step function we can use the same argument we used in the proof of Theorem \ref{teo:sigmoide}.

The first step is therefore to study the limit of equation (\ref{eq:limite_rete_tanh}). Let us focus of the multiplication of $\tanh$ in the first layer, given by equation (\ref{eq:1layer_tanh}).

The pointwise limit of $h_{\lambda}(x)$ for $\lambda \to \infty$ depends on the sign of the limit of the product of $\tanh$, that in turn depends on the sign of $w_k(x-b_k) $ for $k  \in \{1,..., m\}$.

\begin{remark}
\label{remark_ortanti}
We remark that for $x \in [0,1]$, from the limit of equation (\ref{eq:1layer_tanh}) we can obtain the indicator functions of set of the form $x>b$ or $x<b$ for any $b \in \mathbb{R}$.
We just have to choose the parameters in such a way that only one of the $\tanh$ in the multiplication is relevant. Let us define $Z=\{k \in \{1,...,m \}: w_{k}(x -b_{k}) >0 \quad \forall x \in[0,1] \}$.
If $|Z|=m-1$, i.e. there is only one $i \in \{1,...,m\}$ so that its weight are significant it holds that
$$\lim_{\lambda \to \infty} h_{\lambda}(x) = \upsilon(x)=
\begin{cases}
    1 & \text{ if } w_{i}(x-b_{i})>0 \\
    \sigma(2\phi) & \text{ if }  w_{i}(x-b_{i})=0\\
    0 & \text{ if } w_{i}(x-b_{i})<0
    \end{cases}
$$

taking into account that $\sigma(2\phi) = \frac{1}{2}\left(\tanh(\phi)+1\right)$.
\end{remark}

\begin{proof}[Proof of Theorem \ref{teo:tanh}]
Considering Remark \ref{remark_ortanti}, the proof of Theorem \ref{teo:tanh} is analogous to that of Theorem \ref{teo:sigmoide}.
\end{proof}


\section{Experimental settings}

All code was written in Python Programing Language. In particular, the following libraries were used for the algorithms: tensorflow for neural networks, scikit-learn for Logistic Regression, Decision Trees and Gradient Boosting Decision Trees.

A small exploration was made to determine the best structure of the neural network for each dataset. We used a breadth-first search algorithm defined as follows. We started with a network with just one neuron, we trained it and evaluated its performance. At each step, we can double the number of neurons in each layer except the output one or increase the depth of the network by adding a layer with one neuron. For each new configuration, we build a new structure based on it, initialize it and train it. If the difference between the accuracy achieved by the new structure and that of the previous step is lower than $1\%$, then a patience parameter is reduced by $1$. The patience parameter is initialized as $5$ and is passed down from a parent node to its spawned children, so that each node has its own instance of it. When patience reach $0$, that configuration will not spawn new ones. 

Before the neural network initialization, a random seed was set in order to reproduce the same results. As for the initialization of \ian, the weights $w$ are initialised using the glorot uniform. For the biases $b$ of the first layer a uniform between the minimum and the maximum of each feature was used, while for the following layers a uniform between the minimum and the maximum possible output from the neurons of the previous layer was used.

For the network training, Adam with a learning rate equal to $0.1$ was used as optimization algorithm. The loss used is the binary or categorical crossentropy, depending on the number of classes in the dataset. In the calculation of the loss, the weight of each class is also taken into account, which is inversely proportional to the number of samples of that class in the training set.
The maximum number of epochs for training has been fixed at $10000$. To stop the training, an early stopping method was used based on the loss calculated on the train. The patience of early stopping is $250$ epochs, with the variation that in these epochs the loss must decrease by at least $0.01$. Not using a validation dataset may have led to overfitting of some structures, so in future work we may evaluate the performance when using early stopping based on a validation loss. The batch size was fixed at $128$ and the training was performed on CPU or GPU depending on which was faster considering the amount of data. The Heaviside was trained as if its derivative was the same as the sigmoid.

For Decision Trees (DT) and Gradient Boosting Decision Trees (GBDT), an optimisation of the hyperparameters was carried out, in particular for min\_samples\_split (between 2 and 40) and min\_samples\_leaf (between 1 and 20). For GBDT, $1000$ estimators were used, while for DT the class\_weight parameter was set. For the rest of the parameters, we kept the default values of the python sklearn library.

\section{Datasets}
$19$ out of $23$ datasets are publicly available, either on the UCI Machine Learning Repository website or on the Kaggle website. Here we present a full list of the datasets used, correlated with their shortened and full-lenght name, and the corresponding webpage where the description and data can be found.

\begin{table*}[!ht]
\centering
\begin{tabular}{l|l|l}
Short name & Full-length name                             & {\color[HTML]{333333} Webpage}                                                                                    \\ \hline
adult      & Adult                                        & $<$UCI\_MLR\_URL$>$/adult                                        \\
australian & Statlog (Australian Credit Approval)         & $<$UCI\_MLR\_URL$>$/statlog+(australian+credit+approval)         \\
b-c-w      & Breast Cancer Wisconsin                      & $<$UCI\_MLR\_URL$>$/Breast+Cancer+Wisconsin+(Diagnostic)         \\
car        & Car Evaluation                               & $<$UCI\_MLR\_URL$>$/car+evaluation                               \\
cleveland  & Heart Disease                                & $<$UCI\_MLR\_URL$>$/heart+disease                                \\
crx        & Credit Approval                              & $<$UCI\_MLR\_URL$>$/credit+approval                              \\
diabetes   & Diabetes                                     & https://www.kaggle.com/uciml/pima-indians-diabetes-database                         \\
german     & Statlog (German Credit Data)                 & $<$UCI\_MLR\_URL$>$/statlog+(german+credit+data)                 \\
glass      & Glass Identification                         & $<$UCI\_MLR\_URL$>$/glass+identification                         \\
haberman   & Haberman's Survival                          & $<$UCI\_MLR\_URL$>$/haberman\%27s+survival                       \\
heart      & Statlog (Heart)                              & $<$UCI\_MLR\_URL$>$/statlog+(heart)                             \\
hepatitis  & Hepatitis                                    & $<$UCI\_MLR\_URL$>$/hepatitis                                    \\
image      & Statlog (Image Segmentation)                 & $<$UCI\_MLR\_URL$>$/Statlog+(Image+Segmentation)                 \\
ionosphere & Ionosphere                                   & $<$UCI\_MLR\_URL$>$/ionosphere                                   \\
iris       & Iris                                         & $<$UCI\_MLR\_URL$>$/iris                                         \\
monks-1    & MONK's Problems                              & $<$UCI\_MLR\_URL$>$/MONK\%27s+Problems                           \\
monks-2    & MONK's Problems                              & $<$UCI\_MLR\_URL$>$/MONK\%27s+Problems                          \\
monks-3    & MONK's Problems                              & $<$UCI\_MLR\_URL$>$/MONK\%27s+Problems                         \\
sonar      & Connectionist Bench & $<$UCI\_MLR\_URL$>$/Connectionist+Bench+(Sonar,+Mines+vs.+Rocks)
\end{tabular}
\caption{Publicly available datasets, with the short name used in in our work, their full-lenght name and the webpage where data and description can be found. The UCI\_MLR\_URL is the following: https://archive.ics.uci.edu/ml/datasets/}
\end{table*}

The $4$ synthetic datasets of our own creation are composed of $1000$ samples with $2$ variables generated as random uniforms between $-1$ and $1$ and an equation dividing the space into $2$ classes.
The $4$ equations used are:
\begin{itemize}
    \item bisector: $x_1>x_2$
    \item xor: $x_1>0 \land x_2>0$
    \item parabola: $x_2 < 2x_1^2 - \frac{1}{2}$
    \item circle $x_1^2 + x_2^2 < \frac{1}{2}$
\end{itemize}

These datasets are also represented in Figure \ref{custom_datasets}.

\begin{figure}[ht]
    \centering
    \begin{subfigure}[h]{0.475\columnwidth}
        \centering
        \includegraphics[width=\textwidth]{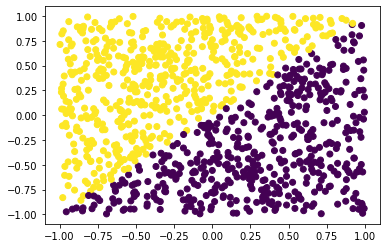}
        \caption{Bisector}    
    \end{subfigure}
    \hfill
    \begin{subfigure}[h]{0.475\columnwidth}  
        \centering 
        \includegraphics[width=\textwidth]{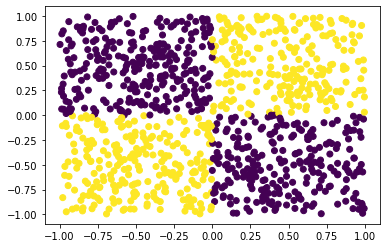}
        \caption{XOR}
    \end{subfigure}
    \vskip\baselineskip
    \begin{subfigure}[h]{0.475\columnwidth}   
        \centering 
        \includegraphics[width=\textwidth]{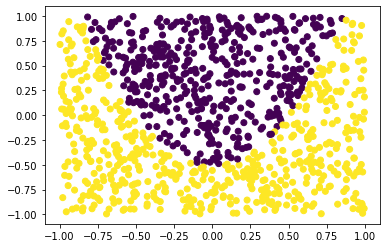}
        \caption{Parabola}  
    \end{subfigure}
    \hfill
    \begin{subfigure}[h]{0.475\columnwidth}   
        \centering 
        \includegraphics[width=\textwidth]{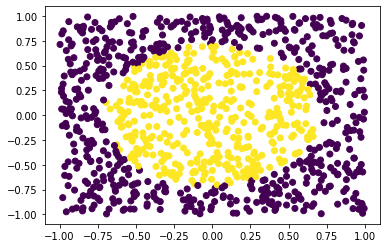}
        \caption{Circle}   
    \end{subfigure}
    \caption{The synthetically generated datasets we used to assess the soundness of our methodology.} 
    \label{custom_datasets}
\end{figure}

\section{Examples}

\subsection{Heart dataset - Heaviside \ian}

\begin{figure*}[!ht]
    \centering
    \includegraphics[width=0.99\textwidth]{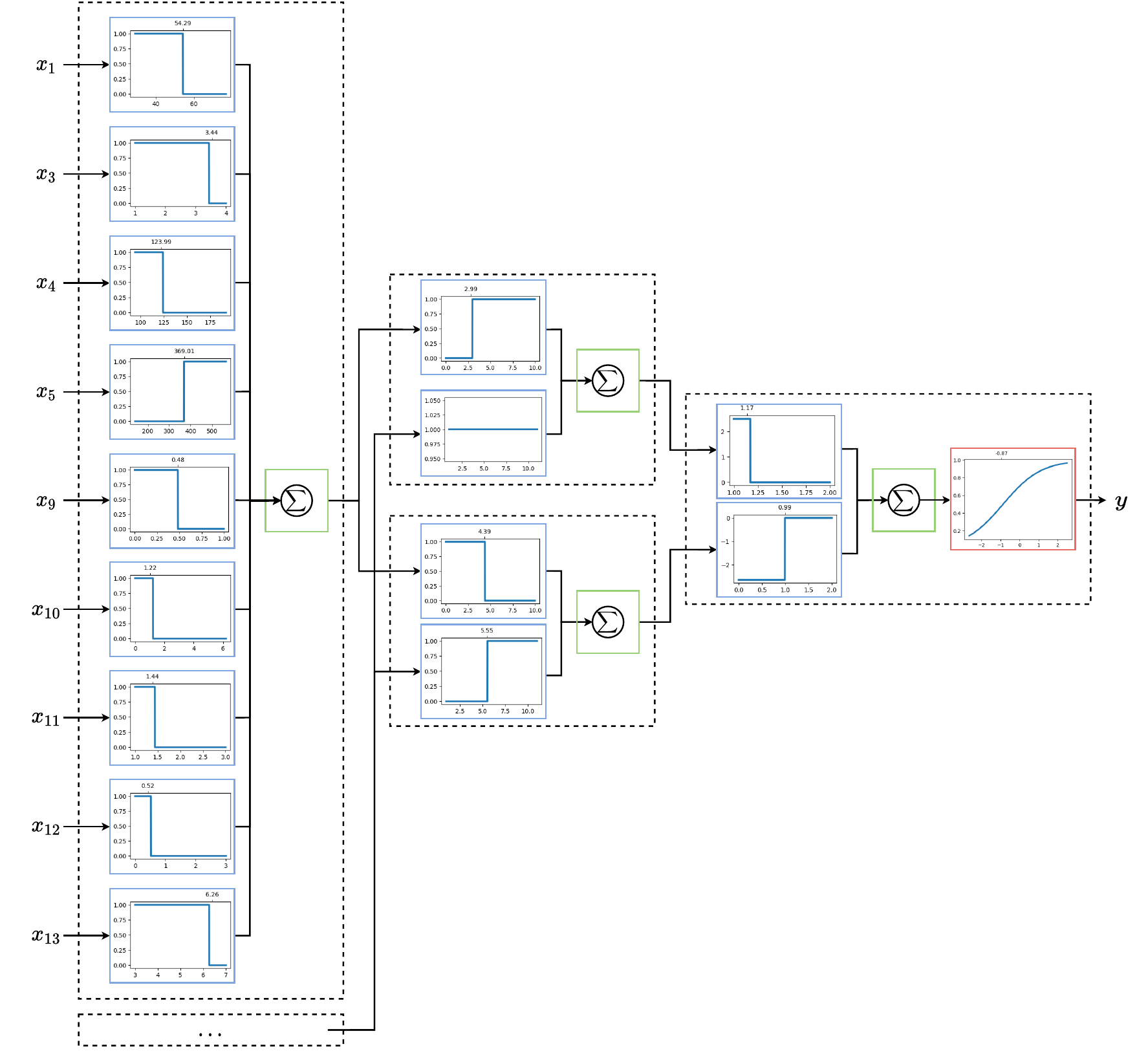}
    \caption{The Heaviside \ian\ Network trained on the heart dataset. The Figure follows the color convention used for \newron.}
    \label{heart_heaviside}
\end{figure*}

The Statlog Heart dataset is composed of $270$ samples and $13$ variables of medical relevance. The dependent variable is whether or not the patient suffers from heart disease.
In Figure \ref{heart_heaviside} you can find the network based on Heaviside \ian\ trained on the heart dataset. Only the inputs with a relevant contribution to the output are shown. From now on, we will indicate with $R_{k,j,i}$ the rule related to the processing function corresponding to the $i$-th input, of the $j$-th neuron, of the $k$-th layer.
From the first neuron of the first layer we can easily retrieve the following rules: $R_{1,1,1} = x_1 \leq 54.29, R_{1,1,3} = x_3 \leq 3.44, R_{1,1,4} = x_4 \leq 123.99, R_{1,1,5} = x_5 \geq 369,01, R_{1,1,9} = x_9 \leq 0.48, R_{1,1,10} = x_{10} \leq 1.22, R_{1,1,11} = x_{11} \leq 1.44, R_{1,1,12} = x_{12} \leq 0.52, R_{1,1,13} = x_{13} \leq 6.26$. The second neuron of the first layer is not shown for lack of space, but its obtained rules are $R_{1,2,2} = x_2 \geq 0.79, R_{1,2,3} = x_3 \geq 3.59, R_{1,2,4} = x_4 \geq 99.95, R_{1,2,5} = x_5 \geq 253.97, R_{1,2,8} = x_8 \leq 97.48, R_{1,2,9} = x_9 \leq 0.04, R_{1,2,10} = x_{10} \geq 2.56, R_{1,2,11} = x_{11} \geq 1.53, R_{1,2,12} = x_{12} \geq 0.52, R_{1,2,13} = x_{13} \geq 5.47$. Moreover, input $x_7$ gives always $1$, so this must be taken into consideration in the next layer.

Moving on to the second layer, we can see in the first neuron that the second input is irrelevant, since the Heaviside is constant. The first processing function activates if it receives an input that is greater or equal to $2.99$. Given that the input can only be an integer, we need at least $3$ of the rules obtained for the first neuron of the first layer to be true: $R_{2,1,1} = 3-of-\{R_{1,1,i}\}$. Following the same line of reasoning, in the second neuron of the second layer we see that we get $R_{2,2,1} = 5-of-\{\neg R_{1,1,i}\}$ and $R_{2,2,2} = 5-of-\{R_{1,2,i}\}$ ($5$ and not $6$ because of $x_7$ processing function).
 
In the last layer, the first processing function has an activation of around $2.5$ if it receives an input that's less than $1.17$. This can happen only if $R_{2,1,1}$ does not activate, so we can say: $R_{3,1,1} = \neg R_{2,1,1} = 7-of-\{\neg R_{1,1,i}\}$. The second processing function gives a value of around $-2.5$ only if it gets an input less than $0.99$, so only if the second neuron of the second layer does not activate. This means that $R_{2,2,1}$ and $R_{2,2,2}$ must be both false at the same time, so we get $R_{3,1,2} = \neg R_{2,2,1} \land \neg R_{2,2,2} = 5-of-\{R_{1,1,i}\} \land 6-of-\{\neg R_{1,2,i}\}$.
Now there are $4$ cases for the sum, i.e. the combinations of the 2 activations: $\{0+0, 2.5+0, 0-2.5, 2.5-2.5\} = \{-2.5, 0, 2.5\}$. Given that both have around the same value for the $\alpha$ parameter, the set is reduced to two cases. Looking at the processing function, we can see that is increasing with respect to the input, so since $\alpha_1$ is positive, we can say that rule $R_{3,1,1}$ is correlated to class $1$, while rule $R_{3,1,2}$, having a negative $\alpha_2$, has an opposite correlation. Looking at its values, we can see that for both $0$ and $2.5$ inputs, the activation function gives an output greater than $0.5$. If we consider this as a threshold, we can say that only for an input of $-2.5$ we get class $0$ as prediction. This happens only if $R_{3,1,2}$ is true and $R_{3,1,1}$ is false. Summarizing we get $R_0 = R_{3,1,2} \land \neg R_{3,1,1} = 5-of-\{R_{1,1,i}\} \land 6-of-\{\neg R_{1,2,i}\} \land 3-of-\{R_{1,1,i}\} = 5-of-\{R_{1,1,i}\} \land 6-of-\{\neg R_{1,2,i}\}$, so that we can say ``if $R_0$ then predicted class is $0$, otherwise is $1$''.

Although we are not competent to analyse the above results from a medical perspective, it is interesting to note for example that the variables $x_1$ and $x_4$, representing age and resting blood pressure respectively, are positively correlated with the presence of a heart problem.

\subsection{Xor - sigmoid \ian}

\begin{figure*}[ht]
    \centering
    \includegraphics[width=0.99\textwidth]{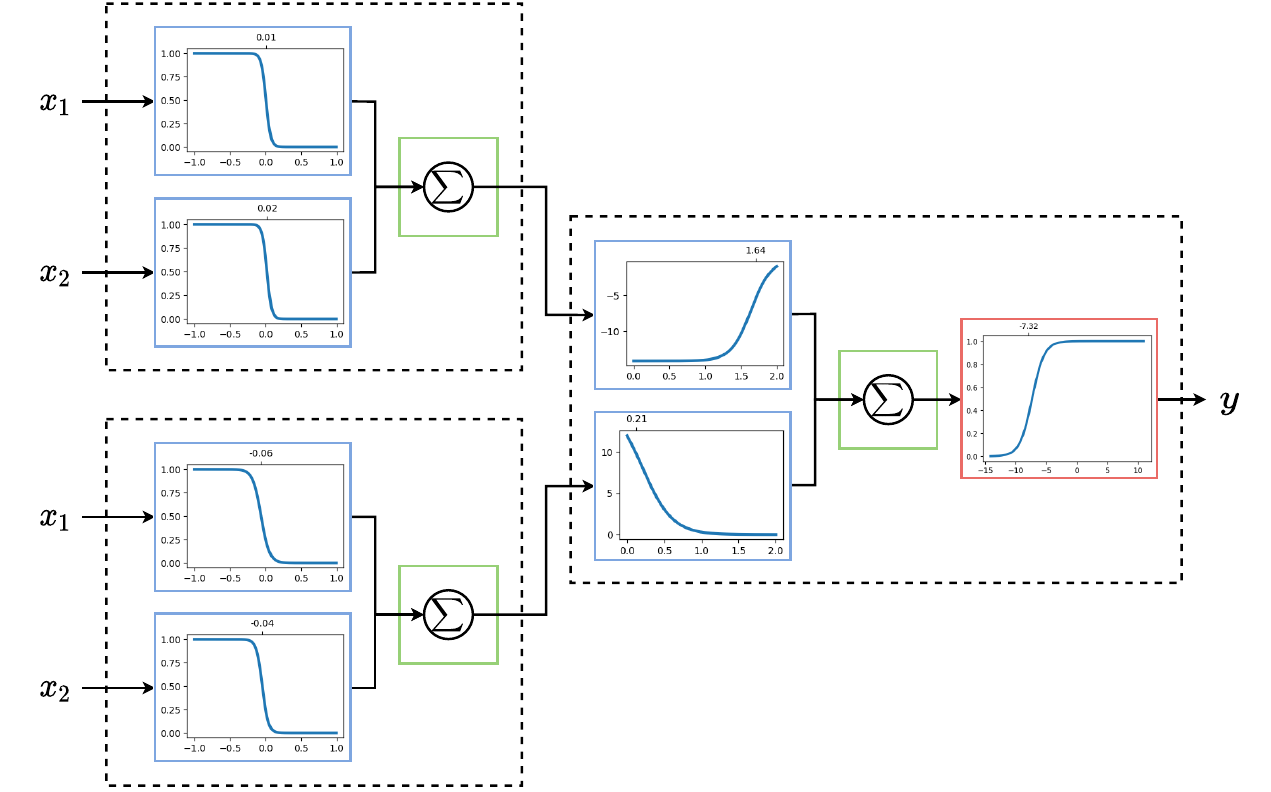}
    \caption{The sigmoid \ian\ Network trained on the xor dataset. The Figure follows the color convention used for \newron.}
    \label{xor_sigmoid}
\end{figure*}

Our custom xor dataset divides the 2D plane in quadrants, with the opposites having the same label.

The network based on sigmoid \ian\ trained on xor dataset is represented in Figure \ref{xor_sigmoid}. As we can see, all the processing functions of the first layer converged to nearly the same shape: a steep inverted sigmoid centered in $0$. Therefore, we can say the rules obtained are $R_{1,1,1} = R_{1,2,1} = x_1 \leq 0$ and $R_{1,1,2} = R_{1,2,2} = x_2 \leq 0$.
In the last layer, the first processing function has a value of about $-15$ for inputs in $[0,1]$, then it starts growing slowly to reach almost $0$ for an input of $2$. This tells us that it doesn't have an activation if both rules of the first neuron are true, so if $x_1 \leq 0 \land x_2 \leq 0$.
On the other hand, the second processing function has no activation if its input greater than $1$, that happens for example if we have a clear activation from at least one of the inputs in the second neuron of the first layer. So looking at it the opposite way, we need both those rules to be false ($x_1>0 \land x_2>0$) to have an activation of $12.5$.
The activation function is increasing with respect to the input, and to get a clear class $1$ prediction, we need the input to be at least $-5$. Considering if the processing functions could give only $\{-15,0\}$ and $\{12.5,0\}$ values, just in the case we got $-15$ from the first one and $0$ from the second one ot would give us a clear class $0$ prediction. This happens only if $\neg (x_1 \leq 0 \land x_2 \leq 0) = x_1>0 \lor x_2>0$ and $\neg (x_1>0 \land x_2>0) = x_1 \leq 0 \lor x_2 \leq 0$, that can be summarised $(x_1>0 \lor x_2>0) \land (x_1 \leq 0 \lor x_2 \leq 0) = (x_1>0 \land x_2\leq0) \land (x_1 \leq 0 \lor x_2 > 0)$. Since this rule describes the opposite to xor, for class $1$ we get the exclusive or logical operation.

\subsection{Iris dataset - $\tanh$-prod \ian}

\begin{figure*}[ht]
    \centering
    \includegraphics[width=0.99\textwidth]{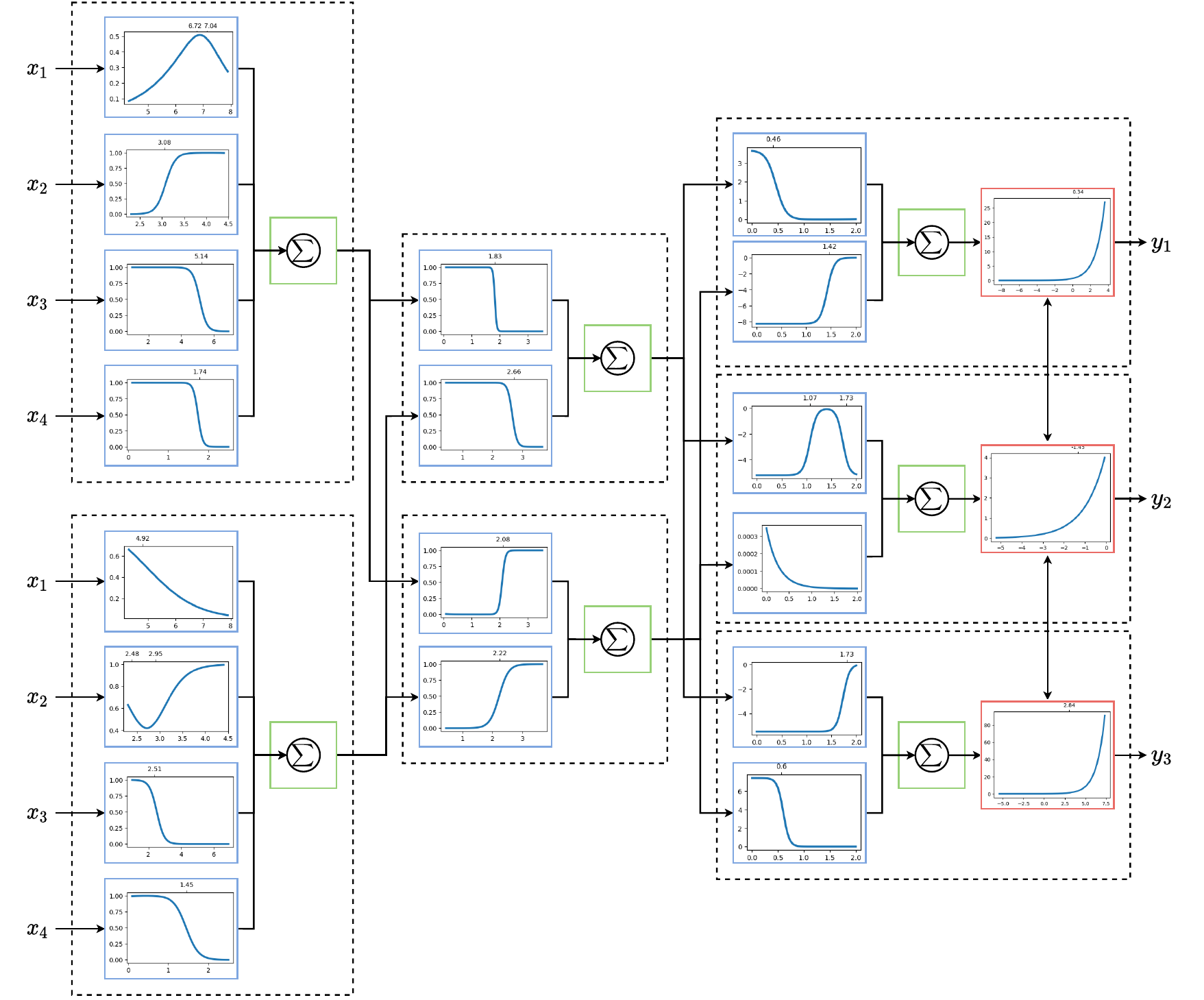}
    \caption{The $\tanh$-prod \ian\ Network trained on the iris dataset. The Figure follows the color convention used for \newron.}
    \label{iris_network}
\end{figure*}

A dataset widely used as a benchmark in the field of machine learning is the Iris dataset. This contains 150 samples, divided into 3 classes (setosa, versicolor and virginica) each representing a type of plant, while the 4 attributes represent in order sepal length and width and petal length and width.

In Figure \ref{iris_network} you can see the final composition of the network generated with the $\tanh$-prod2 IAN neuron.

Considering the first neuron of the first layer, we see that it generates the following fuzzy rules: $R_{1,1,2} = x_2 > 3.08$ (sepal width), $R_{1,1,3} = x_3 < 5.14$ (petal length) and $ R_{1,1,4} = x_4< 1.74$ (petal width). For the first attribute (sepal length) it does not generate a clear rule, but forms a bell shape, reaching a maximum of $0.5$. This tells us that $x_1$ is less relevant than the other attributes, since, unlike the other processing functions, it does not reach $1$.

The second neuron has an inverse linear activation for the first attribute, starting at $0.7$ and reaching almost $0$. The second attribute also has a peculiar activation, with an inverse bell around $2.8$ and a minimum value of $0.4$. The third and fourth attributes have clearer activations, such as $R_{1,2,3} = x_3 < 2.51$ and $R_{1,2,4} = x_4< 1.45$.

The fact that petal length and width are the ones with the clearest activations and with those specific thresholds are in line with what has previously been identified on the Iris dataset by other algorithms.

We denote by $y_{k,j}$ the output of the $j$-th neuron of the $k$-th layer. Moving on to the second layer, the first neuron generates the rules ``if $y_{1,1} < 1.83$'' and ``if $ y_{1,2} < 2.66$'', while the second one generates ``if $y_{2,1} > 2.08$'' and ``if $y_{2,2} > 2.22$''.
Combined with what we know about the previous layer, we can deduce the following: $y_{1,1}$ is less than $1.83$ only if the sum of the input activation functions is less than $1.83$, which only happens if no more than one of the last three rules is activated ($0 + 1 + 0 < 1.83$), while the first one, even taking its maximum value, is discriminative only when the input of one of the other rules is close to the decision threshold ($0.5 + 1 + 0 + 0 < 1.83$, while $0.5 + 1 + 0.5 + 0 > 1.83$). For $y_{1,2} < 2.66$, there are more cases. We can divide the second processing function of the second neuron of the first layer in two intervals: one for which $x_2<3.2$ and the other when $x_2\geq3.2$. In the first interval, the processing function gives a value that is less than $0.66$, greater in the second one. With this, we can say that $y_{1,2} < 2.66$ even if $R_{1,2,3}$ and $R_{1,2,4}$ activates, if $x_2<3.2$ and $x_1$ is near its maximum.
 
In the second neuron of the second layer, the first processing function is nearly the exact opposite to that of the other neuron; we need at least two of $R_{1,1,2}$, $R_{1,1,3}$ or $R_{1,1,4}$ to be true, while $R_{1,1,1}$ still doesn't have much effect. The second processing function gives us $y_{1,2}>2.22$. Considering that the minimum for the processing function related to $x_2$ is $0.4$, we may need both rules $R_{1,2,3}$ and $R_{1,2,4}$ to be true to exceed the threshold, or just one of them active and $x_1$ to take on a low value and $x_2$ to be a high value.

For the last layer, remember that in this case since there are more than $2$ classes, a softmax function is used to calculate the output probability, hence the arrows in the figure that join the layers of the last layer.

For the first output neuron, in order to obtain a clear activation, we need the first input to be less than $0.46$ and the second greater than $1.42$. This is because the $\alpha_i$ are $3$ and $-8$, and the output activation function starts to have an activation for values greater than $-2$. This means that the first neuron of the second layer should hardly activate at all, while the other should activate almost completely. Considering the thresholds for $y_{1,1}$ and $y_{1,2}$, we need the first to be greater than $2.08$ and the other to be greater than $2.66$. So $R_{3,1,1} = 2-of-\{x_2>3.08, x_3<5.14,x_4<1.74\}$. For $R_{3,1,2}$ is more tricky to get a clear decision rule, but we can say that we may need both $R_{1,2,3}$ and $R_{1,2,4}$ to be true and $x_2\geq 3.2$. If $x_2<3.2$, we need $x_1$ to not be near its maximum value. If just one of those two rules is true, we need $x_2<3.2$ and $x_1$ near $4$, or $x_2>3.2$ but with a (nearly) direct correlation with $x_1$, such that the more $x_1$ increases, the same does $x_2$.

In the second output neuron, the second processing function is negligible, while the first one forms a bell shape between $1$ and $2$. This means that it basically captures when $y_{2,1}$ has a value of approximately $1.5$, so when the decision is not clear. This is what gives this neuron maximum activation.

In the third and last output layer, since the first  processing function has a negative $\alpha$ parameter and the activation function is increasing with respect to the input, we want it to output $0$, and this requires maximum activation for the first neuron of the second layer. Regarding the second processing function, we want it to output $8$, so we need nearly no activation from the second neuron of the second layer. So we need the first neuron of the first layer to output a value lower than $1.83$ and the second neuron to output a value lower than $2.22$. This means that no more than one rule $R_{1,1,i}$ needs to be active and at most two rules of $R_{1,2,i}$ need to be true.

We can conclude by saying that both neurons of the first layer are positively correlated with class $1$, while they are negatively correlated with class $3$. This means that low values of $x_3$ and $x_4$, or high values of $x_2$ increase the probability of a sample to belong to class $1$, while $x_1$ has almost no effect. For class $2$, what we can say is that it correlates with a non-maximum activation of both neurons of the first layer, meaning that it captures those cases in which the prediction of one of the other classes is uncertain.

\end{document}